\documentclass[lettersize,journal]{IEEEtran}
\usepackage{amsthm,amsmath,amssymb}
\usepackage{algorithmic}
\usepackage{algorithm}
\usepackage{array}
\usepackage[caption=false,font=normalsize,labelfont=sf,textfont=sf]{subfig}
\usepackage{textcomp}
\usepackage{stfloats}
\usepackage{url}
\usepackage{verbatim}
\usepackage{graphicx}
\usepackage{cite}
\usepackage{multirow}
\usepackage{multicol}
\usepackage{parallel}
\usepackage{booktabs}
\usepackage{threeparttable}
\usepackage{pifont} 
\usepackage{siunitx} 
\usepackage[figuresright]{rotating} 

\newtheorem{prop}{Proposition}
\newtheorem{theorem}{Theorem}
\newtheorem{lemma}{Lemma}
\newtheorem{remark}{Remark}

\usepackage{xcolor}
\newcommand{\tx}{\mathrm}

\RequirePackage[normalem]{ulem} 
\RequirePackage{color}\definecolor{RED}{rgb}{1,0,0}\definecolor{BLUE}{rgb}{0,0,1} 
\RequirePackage{listings} 
\RequirePackage{color} 
\lstdefinelanguage{DIFcode}{ 
  moredelim=[il][\color{red}\sout]{\%DIF\ <\ }, 
  moredelim=[il][\color{blue}]{\%DIF\ >\ } 
} 
\lstdefinestyle{DIFverbatimstyle}{ 
	language=DIFcode, 
	basicstyle=\ttfamily, 
	columns=fullflexible, 
	keepspaces=true 
} 
\lstnewenvironment{DIFverbatim}{\lstset{style=DIFverbatimstyle}}{} 
\lstnewenvironment{DIFverbatim*}{\lstset{style=DIFverbatimstyle,showspaces=true}}{} 

\begin{document}

\title{Efficient optimization-based trajectory planning}

\author{Jiayu Fan,
        Nikolce Murgovski,
        Jun Liang
\thanks{This work is supported by the National Key Research and Development Program of China under Grant 2019YFB1600500, the Chalmers Area of Advance Transport, project EFFECT, and the China Scholarship Council under Grant 202206320304. }%
\thanks{Jiayu Fan and Jun Liang are with the College of Control Science and Engineering, Zhejiang University, Hangzhou 310027, China (e-mail: jiayu.fanhsz@gmail.com; jliang@zju.edu.cn).}%
\thanks{Nikolce Murgovski is with the Department of Electrical Engineering, Chalmers University of Technology, 41296 G\"oteborg, Sweden (e-mail: nikolce.murgovski@chalmers.se).}
\thanks{*Corresponding author (e-mail:jliang@zju.edu.cn).}%
}

\maketitle

\begin{abstract}
This research addresses the increasing demand for advanced navigation systems capable of operating within confined surroundings. A significant challenge in this field is developing an efficient planning framework that can generalize across various types of collision avoidance missions. Utilizing numerical optimal control techniques, this study proposes a unified optimization-based planning framework to meet these demands. We focus on handling two collision avoidance problems, i.e., the object not colliding with obstacles and not colliding with boundaries of the constrained region. The object or obstacle is denoted as a union of convex polytopes and ellipsoids, and the constrained region is denoted as an intersection of such convex sets. Using these representations, collision avoidance can be approached by formulating explicit constraints that separate two convex sets, or ensure that a convex set is contained in another convex set, referred to as separating constraints and containing constraints, respectively. We propose to use the hyperplane separation theorem to formulate differentiable separating constraints, and utilize the S-procedure and geometrical methods to formulate smooth containing constraints. We state that compared to the state of the art, the proposed formulations allow a considerable reduction in nonlinear program size and geometry-based initialization in auxiliary variables used to formulate collision avoidance constraints. Finally, the efficacy of the proposed unified planning framework is evaluated in two contexts, autonomous parking in tractor-trailer vehicles and overtaking on curved lanes. The results in both cases exhibit an improved computational performance compared to existing methods.
\end{abstract}


\section{Introduction}\label{sec1}
\IEEEPARstart{T}{he} optimization-based planning method holds immense potential in planning motions for robotic systems since it allows to naturally include system dynamics as well as further constraints in formulations. It is known that the optimal control problem (OCP) serves as the basis for model predictive control (MPC). Since MPC requires repeated solutions of the OCP, the computational efficiency strongly depends on the capability of solving the OCP. Thus, for readability, the planning problem is often formulated as an OCP \cite{ref1,ref2,ref3}. This is often followed by transcription of the OCP to a nonlinear program (NLP), which is then solved by robust numerical algorithms using, e.g., off-the-shelf solvers.

When utilizing an optimization-based method to plan trajectories for a controlled object, the most crucial concern revolves around implementing effective collision avoidance, requiring the object not to collide with obstacles and environment boundaries.  Typically, obstacles take the form of nonconvex shapes and controlled objects themselves can be either convex or nonconvex shapes \cite{ref61, ref4, ref62}. Furthermore, given that most controlled objects operate within constrained environments, e.g., vessel docking in limited regions \cite{ref6} or car parking in tight parking lots \cite{ref7}, planning problems involving exact collision avoidance become inherently complex. In general, the planning problems are recognized as NP-hard \cite{ref8, ref9, ref82, ref83, ref63}, and the challenges intensify when controlled objects are expected to accurately and efficiently navigate nonconvex obstacles within confined surroundings. Since central to the optimization method is obtaining a computationally efficient NLP solution, the construction of smooth (at least twice differentiable) cost function and constraints \cite{ref3, ref4, ref6} play a pivotal role. This research focuses on the investigation of explicit and differentiable collision avoidance constraints in fully dimensional domains. 

When it comes to constructing collision avoidance constraints for controlled objects, several typical formulations have been developed. A natural choice is the disjunctive programming method \cite{ref10, ref11}. This method can explicitly model the non-convexity in the OCP and reformulate the OCP as a mixed-integer optimization problem. Earlier implementations have typically used the faces of the obstacles themselves to define convex safe regions. The requirement that a point is outside all obstacles is transformed to the requirement that, for each obstacle, the point must be on the outside of at least one face of that obstacle \cite{ref12, ref13}. For convex obstacles, these conditions are equivalent and have been successfully used to encode obstacle avoidance as a mixed-integer linear program \cite{ref14}. Researchers have also investigated a nonconvex feasible region the vehicle resides in. An active research topic in this area is the efficient description of the nonconvex regions. Preliminary results show that these regions are often characterized by exploiting hyperplane arrangements (interested readers can refer to complementary materials in \cite{ref15, ref16}).  For example, the authors in \cite{ref17, ref18} present a partitioning approach through equally-sized square cells, and then using this approach, they show that the collision avoidance problem becomes a resource allocation problem. Apart from these, in papers \cite{ref19, ref20} binary variables combined with the Big-M method \cite{ref21} are employed to avoid collisions in autonomous overtaking scenarios. However, the mixed-integer formulation requires allocating multiple binary variables to collision constraints, which is computationally expensive \cite{ref22, ref23}. Various measures have been proposed to mitigate this issue. The authors in  \cite{ref10} propose to use hyperplane arrangements associated with binary variables for constraint description in a multi-obstacle environment. An over-approximation region is characterized by reducing to strictly binary formulations at the price of increasing conservatism. To accelerate the computation, the authors in \cite{ref20,ref24,ref25} put forward several approaches to reduce the number of binary variables in their respective case studies. However, finding a general approach has proven difficult. Another approach relies on the relaxation of mixed-integer constraints \cite{ref26,ref27}, but these relaxation strategies may sacrifice the solution's accuracy. 

In many applications, approximation methods typically resort to approximating nonconvex or convex shapes with simplistic shapes like ellipsoids and spheres \cite{ref28, ref29, ref30, ref31}. For instance, a controlled object and an obstacle with polyhedral shapes can be over-approximated by two ellipsoids. Since the distance functions between two ellipsoids have analytical expressions, collision avoidance can be easily achieved by requiring nonnegative distance. However, it can be noted that the movable region where a controlled object resides is reduced by this type of approximate method, so this way often leads to conservative solutions and may result in deadlock maneuvers \cite{ref32}, thus hindering controlled objects from finding feasible and collision-free trajectories in confined environments. Specifically, checking vertices of polygonal objects is also a practical way to formulate collision constraints. The idea is that a pair of convex polygons are not intersecting when their vertices are kept outside each other. Following this idea, the authors in \cite{ref110, ref111, ref112} impose collision constraints on vertices of the vehicle through triangle-area-based criterion in tight parking maneuver. It is noted that the constraints formulated on the triangle-area-based criterion are generally non-smooth and non-differentiable \cite{ref112}, leading to solutions at such points with inferior numerical properties. 

Additionally, the research on collision avoidance between two convex shapes is emphasized since nonconvex shapes can be decomposed into a union of convex shapes \cite{ref3, ref6, ref8, ref64}. However, the decomposition is highly problem-independent and should be carefully selected to optimize computational performance in whatever algorithms that model is used. The effective decomposition ways refer to the reported methods in \cite{ref100}.
Following the decomposition idea, the authors in \cite{ref3} first investigate the nonnegative distance between two convex sets, and then impose differentiable collision constraints by reformulating the distance via dualization techniques. Furthermore, to approach exact collision avoidance between two polytopes, the authors in \cite{ref6} introduce two indicator constraints by exploiting Farkas’ Lemma. Since, for many applications, a measure of the severity of collisions is of interest, the previous distance can be replaced with a signed distance between two convex sets, which can provide a measure of penetration depth in case of collision \cite{ref3, ref4, ref6, ref33, ref60}. Through the strong duality of convex optimization, the requirement of a nonnegative signed distance can be exactly reformulated as smooth collision avoidance constraints. This approach has been applied to various kinds of applications spanning from robot avoidance maneuver \cite{ref34}, vessels docking \cite{ref35}, to autonomous car parking \cite{ref22}. However, this type of dual reformulation between a pair of convex sets, whether using strong duality or Farkas’ Lemma, needs adding quite a few dual variables and constraints, depending on the complexity of convex sets \cite{ref3, ref6}. The resulting growth in problem size can be computationally expensive.

In this study, we propose a comprehensive optimization framework for accurately and efficiently navigating a controlled object in a confined high-dimensional region with obstacles. We model the object or obstacle as a union of convex sets and the constrained region is denoted as an intersection of convex sets. Utilizing the defined representations, the challenges of collision avoidance in the context of ensuring that the controlled object avoids collisions with both obstacles and the boundaries of the constrained region can be reformulated as problems related to the separation of two convex sets and the containment of one convex set within another. In light of the requirement for smooth collision avoidance constraints within the optimization methods, we employ the hyperplane separation theorem to address the first separation problem, while the second containment problem is approached through the application of the S-procedure and geometrical methods. Furthermore, we detailed the delineation of explicit and differentiable collision avoidance constraints, which can be effectively resolved using readily available solvers through gradient- and Hessian-based algorithms. In summary, the contributions of this research can be summarized as follows:
\begin{enumerate}
\item{We define a unified optimization-based planning framework for collision-free planning tasks regarding the controlled objects and obstacles represented by a union of convex polytopes and ellipsoids, and the constrained region represented by an intersection of such convex sets.}
\item{We propose smooth methods that ensure a convex polytope or ellipsoid is contained inside another convex polytope or ellipsoid using S-procedure and geometrical methods. The smoothness property allows the use of generic solvers with gradient- and Hessian-based algorithms.}
\item{We introduce efficient collision avoidance formulations by utilizing the hyperplane separation theorem to separate two convex sets, which admits a considerable reduction in the OCP problem size compared to the existing distance reformulations using dual methods.}
\item{We propose novel initialization of auxiliary variables used to formulate collision avoidance constraints by primal methods. The primal methods can provide a direct geometric interpretation of the auxiliary variables, which is not obvious when using the existing dual methods.}
\end{enumerate}

The methods most closely related to ours are the work of \cite{ref1} and \cite{ref33}, where the authors represent convex shapes using a vertex representation \cite{ref1} and a support function representation \cite{ref33}. Their collision avoidance formulations also show an improved computational performance compared to the dual reformulations. Our formulations differ from \cite{ref1} and \cite{ref33} in that we present a more considerable reduction in problem size, i.e., adding fewer auxiliary variables and fewer constraints. Furthermore, rather than utilizing the distance between objects as \cite{ref1} and \cite{ref33} do, our formulations using the hyperplane separation theorem directly find hyperplanes separating the objects, and allow finding good initial guesses in auxiliary variables. Additionally, we also study the problem regarding how to formulate smooth constraints under the condition of a convex set being a subset of a convex set.

The remainder of this paper is organized as follows. Section \ref{sec2} formulates a unified OCP problem to describe the optimization-based planning framework. In Section \ref{sec3}, notations to be utilized are defined. Then in Section \ref{sec4} separation constraints between a pair of convex sets are introduced using the hyperplane separation theorem. In Section \ref{sec5}, containing constraints are present to ensure a convex set inside another convex set. Section \ref{sec6} illustrates the advantages in reducing problem size and well-initializing auxiliary variables compared to the state of the art. Section \ref{sec7} introduces two applications of the proposed optimization-based planning framework and states the efficacy of the proposed methods compared to existing methods in each application. Finally, Section \ref{sec8} provides the concluding remarks and future plans.

\section{Problem statement}\label{sec2}
In this paper, trajectory planning problems are studied for a controlled object that navigates obstacles within a constrained region.  The system motion is modeled as
\begin{equation}\label{eq1}
\dot{\boldsymbol{\xi}} = f\left(\boldsymbol{\xi},\boldsymbol{u}\right), \ \boldsymbol{\xi}(0)=\boldsymbol{\xi}_{\tx{init}},
\end{equation}
where $\boldsymbol{\xi}$ is the state vector, $\boldsymbol{\xi}_{\tx{init}}$ is the initial state, $\boldsymbol{u}$ is the control input and $f$ is the system dynamics.

The constrained region (canvas) is modeled as an intersection of convex sets, denoted as $\mathcal W$, and the controlled object and obstacles are represented by a union of convex sets\footnote{In the commonly used scenarios, such as car parking \cite{ref3, ref32} and vessel docking \cite{ref35} cases, most nonconvex object can be decomposed into a union of convex polytopes and ellipsoids, so in this study, we focus on convex polytopes and ellipsoids. }, denoted as $\mathcal B$ and $\mathcal O$, respectively. The constrained OCP under study is planning the motion of the object, represented by the set $\mathcal B$, from its initial state $\boldsymbol{\xi_{\tx{init}}}$ to a terminal state $\boldsymbol{\xi_{\tx{final}}}$ while always residing within the constrained region $\mathcal W$, and not colliding with obstacles $\mathcal O$. The continuous-time constrained OCP is formulated as 
\begin{subequations}\label{eq2}
\begin{align}
\underset {\boldsymbol{\xi}, \boldsymbol{u}}{\text{min}} \ & \int_0^{t_\tx{f}}  
 \ell\left(\boldsymbol{\xi}(t), \boldsymbol{u}(t)\right)\tx{d}t  \label{eq2a}\\
\text {s.t.}  \ & \boldsymbol{\xi}(0)=\boldsymbol{\xi_{\tx{init}}},\ \boldsymbol{\xi}(t_\tx{f})=\boldsymbol{\xi_{\tx{final}}}, \label{eq2b}\\
 & \dot{\boldsymbol{\xi}}(t) = f\left(\boldsymbol{\xi}(t),\boldsymbol{u}(t)\right), \label{eq2c}\\
 & \boldsymbol{\xi}(t) \in \mathcal{X},\ \boldsymbol{u}(t) \in \mathcal{U}, \label{eq2d}\\
 & \mathcal B(\boldsymbol{\xi}) \subset \mathcal W, \label{eq2e} \\ 
 & \mathcal B(\boldsymbol{\xi}) \cap \mathcal O=\emptyset, \label{eq2f}
\end{align}
\end{subequations}
where $t_\tx{f}$ is the final time, and $\mathcal{X}$ and $\mathcal{U}$ are the admissible sets of state and control, respectively. The function $ \ell$ represents the time-dependent cost. Constraints (\ref{eq2c})-(\ref{eq2f}) are imposed for $\forall t\in \left[0, t_\tx{f}\right]$. For the sake of clarity, the time dependence of $\mathcal O$ is omitted, but we remark that the proposed methods can directly be applied to problems with moving obstacles.  Constraints (\ref{eq2e}) are imposed to prevent the controlled object from colliding with the constrained region boundaries (faces), referred to as containing constraints, and constraints (\ref{eq2f}) enforces collision avoidance between the object and obstacles, referred to separating constraints. Note that if the problem (\ref{eq2}) is to be solved by generic solvers using gradient-based algorithms, constraints (\ref{eq2e}) and (\ref{eq2f}) should be transformed into smooth constraints in an explicit form. In the following, we outline general propositions about containing constraints (\ref{eq2e}) and separating constraints (\ref{eq2f}) and then exactly reformulate  (\ref{eq2e}) and (\ref{eq2f}) as explicit and differentiable constraints.  

\section{Notations}\label{sec3}
A convex polytope can be represented as a linear matrix inequality 
\begin{equation}\label{eq3}
   \mathcal P(P, \boldsymbol p, \boldsymbol s) = \{ \boldsymbol s \in \mathbb R^N \mid P\boldsymbol s \leq \boldsymbol p \},
\end{equation}
where $P=\left[\boldsymbol p_{1,P}, \ldots, \boldsymbol p_{N_P,P}\right]^{\top} \in \mathbb{R}^{N_P \times N}$, $\boldsymbol{p} \in \mathbb{R}^{N_P}$. Here, $\boldsymbol{p}_{k,P}^{\top}$ and $p_{k}$ denote the {\it k}-th  row vector in $P$ and the {\it k}-th entry in $\boldsymbol{p}$, respectively.  

The symbol $V_P=\left[\boldsymbol{v}_{1,P}, \ldots, \boldsymbol{v}_{N_P,P}\right]\in\mathbb R^{N\times N_P}$ represents a matrix of vertices in polytope $\mathcal P$. Symbol  $\boldsymbol{v}_{i,P}$ describes the position of the {\it i}-th vertex of polytope $\mathcal P$. 

An ellipsoid can be denoted as a quadratic inequality
\begin{equation}\label{eq4}
\mathcal E(E, \boldsymbol e, \boldsymbol s) =\left\{\boldsymbol s \in \mathbb{R}^N \mid\left(\boldsymbol s -\boldsymbol e\right)^{\top} E \left(\boldsymbol s-\boldsymbol e\right) \leq 1\right\},
\end{equation}
where $\boldsymbol{e} \in \mathbb{R}^N$ denotes the center of the ellipsoid and $E \in \mathbb{R}^{N \times N} $ is a positive definite, symmetric matrix defining the shape.

A hyperplane with normal vector $\boldsymbol{\lambda} \in \mathbb{R}^N$ can be defined as 
\begin{equation}\label{eq5}
    \mathcal H(\boldsymbol{\lambda}, \mu)=\left\{\boldsymbol{s} \in \mathbb{R}^N \mid \boldsymbol{\lambda}^\top \boldsymbol{s} = \mu\right\}.
\end{equation}
The hyperplane $\mathcal H(\boldsymbol{\lambda}, \mu)$ divides the whole space $\mathbb{R}^N$ into two closed half-spaces, denoted as 
\begin{equation}\label{eq6}
\mathcal H^{-}(\boldsymbol{\lambda}, \mu)=\left\{\boldsymbol{s} \in \mathbb{R}^N \mid \boldsymbol{\lambda}^\top \boldsymbol{s} \leq \mu\right\},
\end{equation} 
and 
\begin{equation}\label{eq7}
\mathcal H^{+}(\boldsymbol{\lambda}, \mu)=\left\{\boldsymbol{s} \in \mathbb{R}^N \mid \boldsymbol{\lambda}^\top \boldsymbol{s} \geq \mu\right\}.
\end{equation} 

In the following sections, the constrained region is described by an intersection of convex polytopes and ellipsoids, denoted as $\mathcal W=\bigcap_l \mathcal W_l$. The controlled object or obstacle is described using a union of convex polytopes and ellipsoids, denoted as $\mathcal B=\bigcup_m \mathcal B_m$, ${\mathcal O=\bigcup_n \mathcal O_n}$.  Sets $\mathcal W_l, \mathcal B_m, \mathcal O_n$ are either polytopes or ellisposids, described using representations \eqref{eq3} and \eqref{eq4}.

\begin{figure}[!t]
\centering
\subfloat[]{\includegraphics[width=0.1\textwidth]{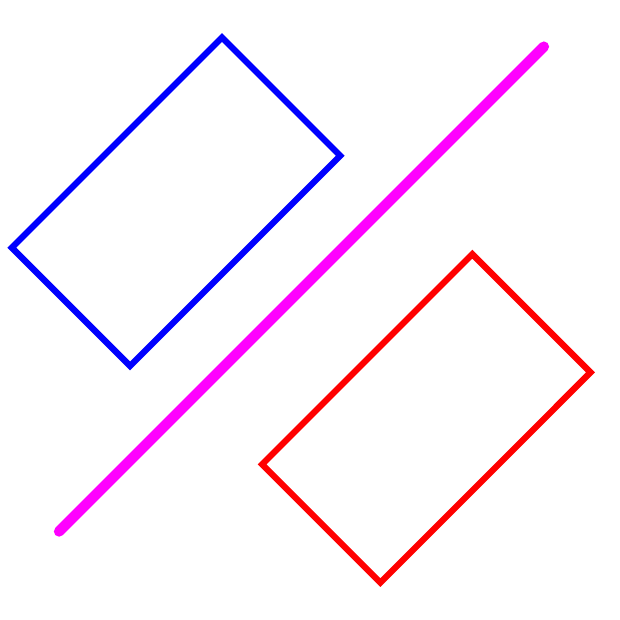}\hspace{2mm}
\label{fig5a}}
\subfloat[]{\includegraphics[width=0.1\textwidth]{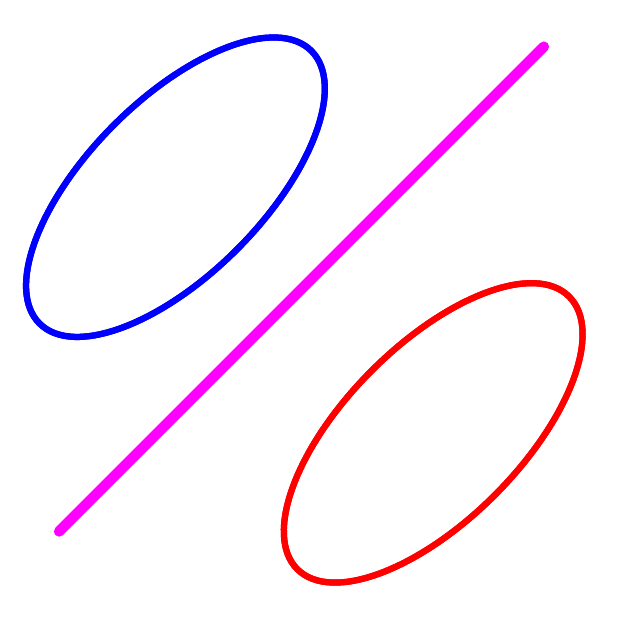}\hspace{2mm}
\label{fig5b}}
\subfloat[]{\includegraphics[width=0.1\textwidth]{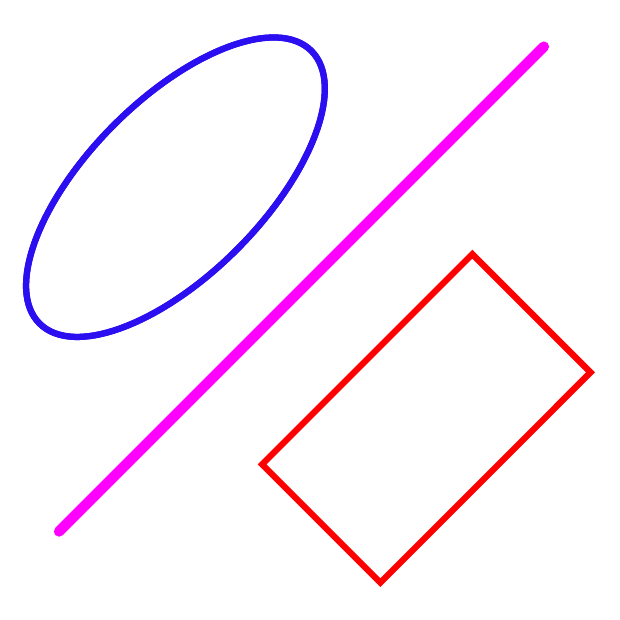}\hspace{2mm}
\label{fig5c}}
\hfill
\subfloat[]{\includegraphics[width=0.1\textwidth]{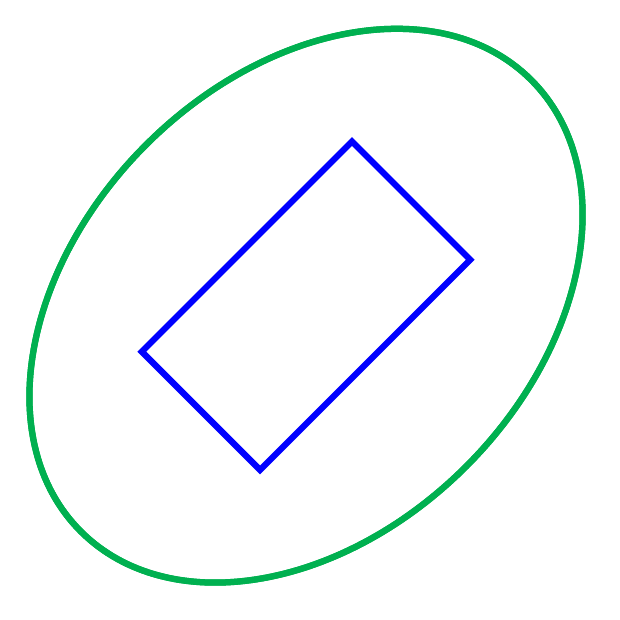}%
\label{fig5d}}
\subfloat[]{\includegraphics[width=0.1\textwidth]{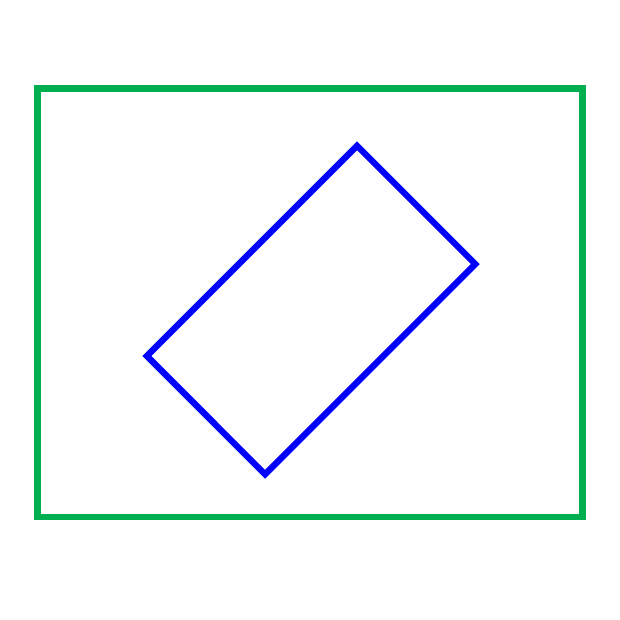}%
\label{fig5e}}
\subfloat[]{\includegraphics[width=0.1\textwidth]{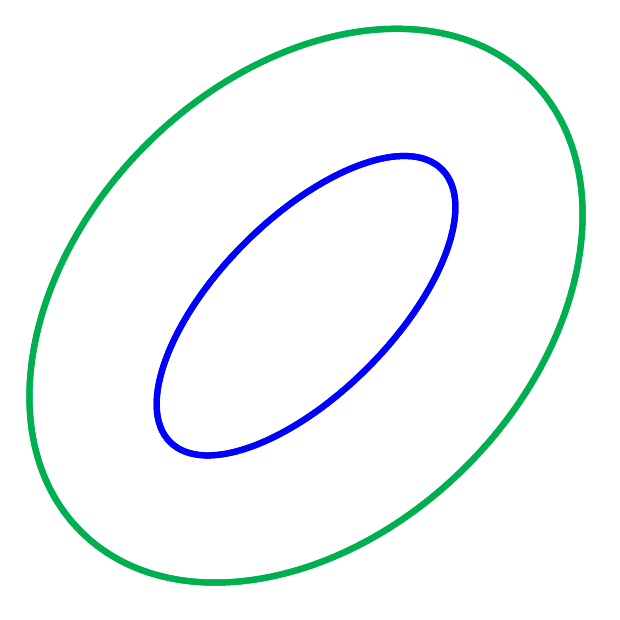}%
\label{fig5f}}
\subfloat[]{\includegraphics[width=0.1\textwidth]{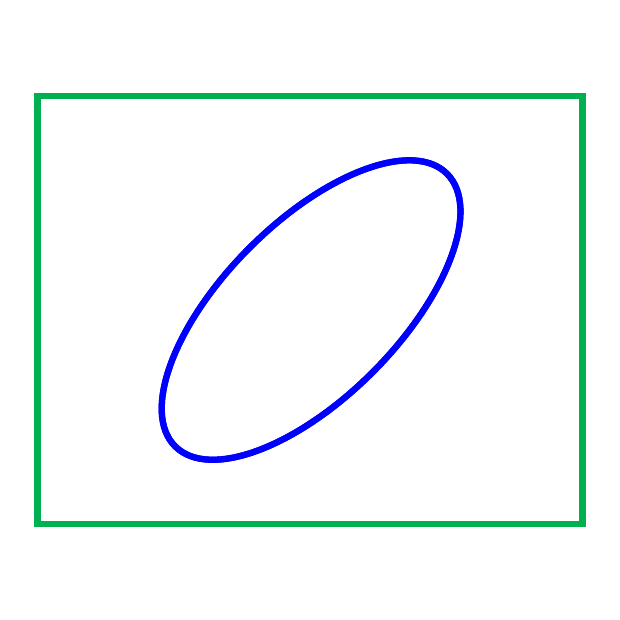}%
\label{fig5g}}
\caption{For readability, this figure illustrates the studied problems in two dimensions as examples. In subplots (a)--(g) the controlled object $\mathcal B_m$ is depicted in blue. In (a)--(c) the obstacle $\mathcal O_n$ is depicted in red. In (d)--(g) the constrained region $\mathcal W_l$ is depicted in green. (a)--(c) illustrate the problem of $\mathcal B_m \cap \mathcal O_n =\emptyset$, i.e., (a) Separation between two polytopes; (b) Separation between two ellipsoids; (c) Separation between a polytope and an ellipsoid. (d)--(g) illustrate the problem of $\mathcal B_m \subset \mathcal W_l$, i.e., (d) Polytope contained in ellipsoid; (e) Polytope contained in polytope; (f) Ellipsoid contained in ellipsoid; (g) Ellipsoid contained in polytope.} 
\label{fig5}
\end{figure}

\section{Separating constraints}\label{sec4}
In this section, we aim at formulating effective collision avoidance constraints between the controlled object and obstacles, denoted as $\mathcal B(\boldsymbol{\xi}) \cap \mathcal O=\emptyset$ in \eqref{eq2}. Since $\mathcal B=\bigcup_m \mathcal B_m$, $\mathcal O=\bigcup_m \mathcal O_n $, separating constraints $\mathcal B(\boldsymbol{\xi}) \cap \mathcal O=\emptyset$ can be represented as
\begin{equation}\label{eq9}
\mathcal B \cap \mathcal O =\emptyset \iff  \mathcal B_m \cap \mathcal O_n =\emptyset, \ \forall m, n.
\end{equation}
As mentioned, set $\mathcal B_m$ or $\mathcal O_n$ can either be a convex polytope or an ellipsoid. To approach \eqref{eq9}, the conditions between two polytopes, between ellipsoids, and between a polytope and an ellipsoid need to be considered, respectively (see Fig.~\ref{fig5}). In the following, we first outline the hyperplane separation theorem and then two general propositions are stated for a polytope contained inside a convex set and an ellipsoid contained inside a half-space. Finally the requirement of $\mathcal B_m \cap \mathcal O_n=\emptyset$ is reformulated as explicit constraints by exploiting the propositions.

\begin{theorem}[Hyperplane separation theorem \cite{ref36}]\label{theorem1}
    Let $\mathcal C^1$ and  $\mathcal C^2$ represent two compact convex sets in $\mathbb R^N$. Then,
    \begin{equation}\label{eq10}
    \begin{aligned}
    \mathcal C^1 \cap \mathcal C^2=\emptyset \iff \exists \boldsymbol{\lambda} \in \mathbb{R}^N \backslash \boldsymbol{0}, \mu \in \mathbb{R}: \ \ \ \ \ \ \ \ \ \  \\
    \forall \boldsymbol{q} \in \mathcal C^1, \forall \boldsymbol{p} \in \mathcal C^2, \boldsymbol{\lambda}^\top\boldsymbol{q} \geq \mu, \boldsymbol{\lambda}^\top \boldsymbol{p} \leq \mu.
    \end{aligned}
    \end{equation} 
\end{theorem}

\begin{prop}[Polytope within a convex set]\label{prop1}
Let a convex set be denoted as $\mathcal C(\Omega, \boldsymbol{\mu}, \mathcal{K})=\left\{\boldsymbol{s} \in \mathbb{R}^N \mid \Omega \boldsymbol{s} \preceq_{\mathcal K} \boldsymbol{\mu}\right\}$\footnote{This form is entirely generic since any compact convex set admits a conic representation. A polytope can be represented by choosing $\mathcal K= \mathbb{R}^N_{+}$.  An ellipsoid is represented by letting $\mathcal K$ be the second-order cone.\cite{ref3,ref37,ref38}} where $\Omega \in \mathbb R^{M \times N}$, $ \boldsymbol{\mu} \in \mathbb R^{M}$, and $\mathcal K \subset \mathbb R^{M}$ is a closed convex pointed cone with nonempty interior. Let $\mathcal V_{\mathcal P}$ represent the set of vertices of polytope $\mathcal P$ denoted as \eqref{eq3}. Then,
\begin{equation}\label{eq11}
\begin{aligned}
\mathcal \mathcal V_{\mathcal P} \subset \mathcal C(\Omega, \boldsymbol{\mu}, \mathcal{K}) \iff \mathcal P \subset \mathcal C(\Omega, \boldsymbol{\mu}, \mathcal{K}).
\end{aligned}
\end{equation}
\end{prop}
\begin{proof}
Since any vertex in $\mathcal V_{\mathcal P}$ belongs to set $\mathcal P$, we have $\mathcal V_{\mathcal P} \subset \mathcal P$. Then, if $\mathcal P \subset \mathcal C(\Omega, \boldsymbol{\mu}, \mathcal{K})$ it follows that $\mathcal V_{\mathcal P} \subset \mathcal C(\Omega, \boldsymbol{\mu}, \mathcal{K})$. 
This proves the right-to-left implication. The left-to-right implication, given $\mathcal V_{\mathcal P} \subset \mathcal C(\Omega, \boldsymbol{\mu}, \mathcal{K})$ is proven by contradiction. Assume $\exists \boldsymbol{s} \in \mathcal P $, $ \boldsymbol{s}\notin \mathcal C(\Omega, \boldsymbol{\mu}, \mathcal{K})$, such that $\Omega \boldsymbol{s} \succ_{\mathcal K} \boldsymbol{\mu}$. It is known that any interior point of a convex polytope can be expressed as a convex combination of its vertices, where all coefficients are non-negative \cite{ref39}. Let $\boldsymbol{s}=\sum_{i=1}^{N_m}\alpha_i \boldsymbol{v}_{i,P}$, where $\alpha_i \geq 0$. Then, $\sum_{i=1}^{N_m}\alpha_i \Omega \boldsymbol{v}_{i,P} \succ_{\mathcal K} \boldsymbol{\mu}$, indicates that $\exists i$ for which $ \Omega \boldsymbol{v}_{i,P} \succ_{\mathcal K} \boldsymbol{\mu}$, i.e., $\boldsymbol{v}_{i,P} \notin \mathcal C(\Omega, \boldsymbol{\mu}, \mathcal{K})$, which contradicts with $\mathcal V_{\mathcal P} \subset \mathcal C(\Omega, \boldsymbol{\mu}, \mathcal{K})$. So, it follows that $\forall \boldsymbol{s} \in \mathcal P$, $\boldsymbol{s} \in \mathcal C(\Omega, \boldsymbol{\mu}, \mathcal{K})$, i.e., $ \mathcal P \subset \mathcal C(\Omega, \boldsymbol{\mu}, \mathcal{K})$, which completes the proof. 
\end{proof}

\begin{prop}[Ellipsoid within a half-space]\label{prop5}
Let an ellipsoid  $\mathcal E$ be denoted as \eqref{eq4}. Let half-space $\mathcal H^{-}(\boldsymbol{\lambda}, \mu)$ be denoted as \eqref{eq6}. Then, 
\begin{equation}\label{eq22}
 \begin{aligned}
\mathcal E \subset \mathcal H^{-}(\boldsymbol{\lambda}, \mu) =\emptyset \iff  \sqrt{\boldsymbol{\lambda}^{\top} E^{-1} \boldsymbol{\lambda}}+\boldsymbol{\lambda}^{\top} \boldsymbol{e} \leq \mu.
 \end{aligned}
\end{equation}
\end{prop}

\begin{proof}
It is known that an ellipsoid can be described as an affine transformation of the unit sphere. Using the representation in \eqref{eq4}, $\mathcal E$ can be re-denoted as $\mathcal E :=\left\{ M\boldsymbol{s}+\boldsymbol{e} \mid \boldsymbol{s} \in \mathbb{R}^N, \Vert \boldsymbol{s} \Vert \leq 1 \right\}$, where $M$ is a positive definite matrix obtained by $M^\top M = E^{-1}$ \cite{ref91}. Thus, the requirement of $\mathcal E \subset \mathcal H^{-}(\boldsymbol{\lambda}, \mu) =\emptyset$ is identical to stating an optimization problem  $$\max_{\boldsymbol{s}} \boldsymbol{\lambda}^\top(M\boldsymbol{s}+\boldsymbol{e}) \leq \mu, \Vert \boldsymbol{s} \Vert \leq 1.$$ According to the well-known Cauchy-Schwartz inequality \cite{ref90}, $\boldsymbol{\lambda}^\top \boldsymbol{s} \leq \Vert \boldsymbol{\lambda} \Vert \Vert \boldsymbol{s} \Vert$. Since $\Vert \boldsymbol{s} \Vert \leq 1$, $\boldsymbol{\lambda}^\top\boldsymbol{s} \leq \Vert \boldsymbol{\lambda} \Vert$, where the right part is obtained by the vector $\boldsymbol{s}=\boldsymbol{\lambda}/\Vert \boldsymbol{\lambda} \Vert$, so $\max_{\boldsymbol{s}} \boldsymbol{\lambda}^\top(M\boldsymbol{s}+\boldsymbol{e})=\Vert \boldsymbol{\lambda} M^\top \Vert  + \boldsymbol{\lambda}^\top \boldsymbol{e}=\sqrt{\boldsymbol{\lambda}^{\top} E^{-1} \boldsymbol{\lambda}}+\boldsymbol{\lambda}^{\top} \boldsymbol{e} \leq \mu$, which completes the proof.
\end{proof}

In the following, we discuss three separation questions, i.e., how to formulate $\mathcal B_m \cap \mathcal O_n =\emptyset$ when set $\mathcal B_m$ is a polytope and $\mathcal O_n$ is a polytope, set $\mathcal B_m$ is an ellipsoid and $\mathcal O_n$ is an ellipsoid, set $\mathcal B_m$ is a polytope and $\mathcal O_n$ is an ellipsoid. 

\subsection{Separation between two polytopes}\label{sec4.1}

According to Theorem~\ref{theorem1}, when compact sets $\mathcal C^1$ and set $\mathcal C^2$ are disjoint, there exists a separating hyperplane between them, i.e., points belonging to $\mathcal C^1$ are on one side of the hyperplane while points of $\mathcal C^2$ are on the other side. However, checking all points of $\mathcal C^1$ and $\mathcal C^2$ is computationally expensive. One efficient way is only checking vertices of $\mathcal C^1$ and $\mathcal C^2$ when these sets are polytopes. 

\begin{prop}\label{prop2}
Let sets $\mathcal B_m$, $\mathcal O_n$ be denoted using polytopes $\mathcal P(P_1, \boldsymbol p_1, \boldsymbol s)$ and $\mathcal P(P_2, \boldsymbol p_2, \boldsymbol s)$ based on \eqref{eq3}. Let $V_{P_1}$ and $V_{P_2}$ be matrices of vertices as stated earlier. Then, 
\begin{equation}\label{eq12}
\begin{aligned}
\mathcal B_m \cap \mathcal O_n=\emptyset \iff  \exists \boldsymbol{\lambda} \in \mathbb{R}^N , \mu \in \mathbb{R}: \ \ \ \ \ \ \ \ \  \ \ \ \   \\ 
  \begin{array}{c}
\boldsymbol{\lambda}^\top V_{P_1} \geq \mu \boldsymbol{1}^\top, \boldsymbol{\lambda}^\top V_{P_2} \leq \mu \boldsymbol{1}^\top, \left\|\boldsymbol{\lambda}\right\| > 0.
\end{array}
\end{aligned}
\end{equation}
\end{prop}

\begin{proof}
According to \eqref{eq3}, sets $\mathcal B_m$ and $\mathcal O_n$ are closed and bounded. Since a closed and bounded set of $R^N$ is compact \cite{ref40}, $\mathcal B_m$ and $\mathcal O_n$ are compact sets. Thus, based on the hyperplane separation theorem $\mathcal B_m \cap \mathcal O_n=\emptyset$ is identical to stating that there exists a separating hyperplane between them.  Let $ \mathcal H^{+}(\boldsymbol{\lambda}, \mu) $ and $\mathcal H^{-}(\boldsymbol{\lambda}, \mu) $ be defined as stated earlier, i.e., $\mathcal H^{-}(\boldsymbol{\lambda}, \mu)=\left\{\boldsymbol{q} \in \mathbb{R}^N \mid \boldsymbol{\lambda}^\top \boldsymbol{q} \leq \mu\right\}$ and $\mathcal H^{+}(\boldsymbol{\lambda}, \mu)=\left\{\boldsymbol{q} \in \mathbb{R}^N \mid \boldsymbol{\lambda}^\top \boldsymbol{q} \geq \mu\right\}$. Assume $\mathcal B_m \subset \mathcal H^{+}(\boldsymbol{\lambda}, \mu) $, $\mathcal O_n \subset \mathcal H^{-}(\boldsymbol{\lambda}, \mu) $. Using Proposition~\ref{prop1}, formulations $\mathcal B_m \subset \mathcal H^{+}(\boldsymbol{\lambda}, \mu) $, $\mathcal O_n \subset \mathcal H^{-}(\boldsymbol{\lambda}, \mu) $ can be reformulated by ensuring that the vertices of $\mathcal B_m$ are in half-space $\mathcal H^{+}(\boldsymbol{\lambda}, \mu)$ and vertices of $\mathcal O_n$ are in the other half-space $\mathcal H^{-}(\boldsymbol{\lambda}, \mu)$, i.e., $\boldsymbol{\lambda}^\top V_{P_1} \geq \mu \boldsymbol{1}^\top, \boldsymbol{\lambda}^\top V_{P_2} \leq \mu \boldsymbol{1}^\top.$ Additionally, $\left\|\boldsymbol{\lambda}\right\| > 0$ is the sufficient and necessary condition of $\boldsymbol{\lambda} \ne \boldsymbol{0}$, ensuring the existence of such a separating hyperplane,  which completes the proof.
\end{proof}

\subsection{Separation between two ellipsoids}\label{sec4.2}
\begin{prop}\label{prop3}
Let sets $\mathcal B_m$, $\mathcal O_n$ be denoted using ellipsoids $\mathcal E(E_1, \boldsymbol e_1, \boldsymbol s)$ and $\mathcal E(E_2, \boldsymbol e_2, \boldsymbol s)$  based on \eqref{eq4}. Then,
\begin{equation}\label{eq13}
\begin{aligned}
\mathcal B_m \cap \mathcal O_n=\emptyset \iff  \exists \boldsymbol{\lambda} \in \mathbb R^N, \mu \in  \mathbb R: \Vert \boldsymbol{\lambda} \Vert > 0,\ \ \ \ \ \ \ \ \ \    \\
\sqrt{\boldsymbol{\lambda}^{\top} E_{1}^{-1} \boldsymbol{\lambda}}+\boldsymbol{\lambda}^{\top} \boldsymbol{e_{1}} \leq \mu, -\sqrt{\boldsymbol{\lambda}^{\top} E_2^{-1} \boldsymbol{\lambda}} + \boldsymbol{\lambda}^{\top} \boldsymbol{e_2} \geq \mu.\\
\end{aligned}
\end{equation}
\end{prop}
\begin{proof}
For the same argumentation as the Proof of Proposition~\ref{prop2}, sets $\mathcal B_m$ and $\mathcal O_n$ are compact, and then using hyperplane separation theorem $\mathcal B_m \cap \mathcal O_n=\emptyset$ is identical to stating they can be separated by a hyperplane. And then Proposition~\ref{prop3} can be approached by using Proposition~\ref{prop5}.
\end{proof}

\subsection{Separation between a polytope and an ellipsoid}\label{sec4.3}

\begin{prop}\label{prop4}
Let set $\mathcal B_m$ be denoted using polytope $\mathcal P(P, \boldsymbol p, \boldsymbol s)$ as \eqref{eq3}. let $V_{P}$ be stated earlier. Let set $\mathcal O_n$ be denoted using ellipsoid $\mathcal E(E, \boldsymbol e, \boldsymbol s)$ as \eqref{eq4}. Then\footnote{The other way around, i.e., set $\mathcal B_m$ is an ellipsoid and set $\mathcal O_n$ is a polytope, works as well for formulation \eqref{eq21}.},
\begin{equation}\label{eq21}
\begin{aligned}
\mathcal B_m \cap \mathcal O_n=\emptyset \iff  \exists \boldsymbol{\lambda} \in \mathbb{R}^N , \mu \in \mathbb{R}: \ \ \ \ \ \ \ \ \  \ \  \\ 
  \begin{array}{c}
\boldsymbol{\lambda}^\top V_P \geq \mu \boldsymbol{1}^\top, \sqrt{\boldsymbol{\lambda}^{\top} E_{1}^{-1} \boldsymbol{\lambda}}+\boldsymbol{\lambda}^{\top} \boldsymbol{e_{1}} \leq \mu, \left\|\boldsymbol{\lambda}\right\| > 0.\\
   \\
\end{array}
\end{aligned}
\end{equation}
\end{prop}
\begin{proof}
Following the proof of Proposition~\ref{prop2} and  Proposition~\ref{prop3}, the hyperplane separation theorem is utilized to formulate Proposition~\ref{prop4}.
\end{proof}

\begin{remark}
Notably, collision formulations related to the ellipsoid representation, e.g., in Proposition~\ref{prop3} and Proposition~\ref{prop4}, involve the square root, which is not actually differentiable at the origin. However, such a singularity in its gradient at the origin is excluded because $\left\|\boldsymbol{\lambda}\right\| > 0$ \cite{ref1, ref33} . In practice, it can be approached by $\left\|\boldsymbol{\lambda}\right\| \geq \epsilon$ , where $\epsilon$ is set as a small value, to ensure $\left\|\boldsymbol{\lambda}\right\|$ strictly positive. Additionally, the proposed geometry-based guesses, covered in Section~\ref{sec4.2}, guarantee that $\boldsymbol{\lambda}$ is initialized sufficiently away from the origin, so numerical solvers do not run into problems. This is also supported by the numerical results in Section~\ref{sec7}.
\end{remark}

\section{Containing constraints}\label{sec5}
In this part, we continue to discuss the explicit form of containing constraints \eqref{eq2e} in problem \eqref{eq2}. Since the controlled object is denoted as a union and the constrained region is represented as an intersection, i.e., $\mathcal B=\bigcup_m \mathcal B_m$, $\mathcal W=\bigcap_l \mathcal W_l$, containing constraints \eqref{eq2e} can be substituted with 
\begin{equation}\label{eq27}
\mathcal B \subset \mathcal W \iff  \mathcal B_m \subset \mathcal W_l, \ \forall m,l, 
\end{equation}
where \eqref{eq27} means that set $B_m$ is a subset of set $\mathcal W_l$. The sole criterion for $\mathcal B_m \subset \mathcal W_l$ is identical to stating that points of set $\mathcal B_m$ belong to set $\mathcal W_l$. However, checking all points of $\mathcal B_m$ is impossible in reality. Next, by using the previous propositions and S-procedure, we explicitly formulate $\mathcal B_m \subset \mathcal W_l$ when $B_m$ and $\mathcal W_l$ are represented as convex polytopes and ellipsoids, respectively (see Fig.~\ref{fig5}).

\subsection{A polytope contained in a polytope or an ellipsoid}\label{sec5.1}

When set $\mathcal B_m$ is a polytope,  $\mathcal B_m \subset \mathcal W_l$ is identical to stating that the vertices of $\mathcal B_m$ are inside $\mathcal W_l$ by using Proposition~\ref{prop1}. Since $\mathcal W_l$ is denoted as either a convex polytope or ellipsoid, the containing constraint $\mathcal B_m \subset \mathcal W_l$ is formulated as a linear inequality or quadratic inequality. 

\begin{prop}\label{prop6}
Let set $\mathcal B_m$ be denoted using polytope $\mathcal P(P_1, \boldsymbol p_1, \boldsymbol s)$ as \eqref{eq3}. Let set $\mathcal W_l$ be denoted using polytope $\mathcal P(P_2, \boldsymbol p_2, \boldsymbol s)$ as \eqref{eq3} or ellipsoid $\mathcal E(E, \boldsymbol e, \boldsymbol s)$ as \eqref{eq4}. Then, $\mathcal B_m \subset \mathcal W_l \iff $
\begin{equation}\label{eq28}
\begin{aligned}
\begin{cases}
    P_2 \boldsymbol{v}_{i,P_1}\leq \boldsymbol{p}_2, \  \forall i & \tx{if} \ \mathcal W_l:= \eqref{eq3}\\
    \left(\boldsymbol{v}_{i,P_1}-\boldsymbol{e}\right)^{\top} E\left(\boldsymbol{v}_{i,P_1}-\boldsymbol{e}\right) \leq 1, \ \forall i & \tx{if} \ \mathcal W_l:= \eqref{eq4}
\end{cases} .
\end{aligned}
\end{equation}
\end{prop}
Proposition~\ref{prop2} shows that $\mathcal B_m \subset \mathcal W_l$, in the condition of a polytope $\mathcal B_m$, can be formulated by imposing constraints on finite vertices of set $\mathcal B_m$. However, when set $\mathcal B_m$ is described by an ellipsoid as \eqref{eq4}, $\mathcal B_m \subset \mathcal W_l$ cannot be ensured by checking vertices of $\mathcal B_m$. To handle this, we propose to formulate $\mathcal B_m \subset \mathcal W_l$ by S-procedure and geometrical methods, respectively. Next, we discuss two questions, i.e., how to formulate $\mathcal B_m \subset \mathcal W_l$ when an ellipsoid $\mathcal B_m$ is a subset of an ellipsoid $\mathcal W_l$ and an ellipsoid $\mathcal B_m$ is a subset of a polytope $\mathcal W_l$, respectively. 

\subsection{An ellipsoid contained in an ellipsoid}\label{sec5.2}

\begin{lemma}[S-procedure \cite{ref42}]\label{lemma1}
 Let $F_1, F_2\in \mathbb R ^{N \times N}$, $\boldsymbol{g_1}, \boldsymbol{g_2} \in \mathbb R ^{N}$, and $h_1, h_2 \in \mathbb R$. If $F_1=F_1^{\top}, F_2=F_2^{\top}$ and $\exists \boldsymbol{x} \in \mathbb R ^{N}$ for which the strict inequality $\boldsymbol{x}^{\top} F_1 \boldsymbol{x}+2 \boldsymbol{g_1}^{\top} \boldsymbol{x}+h_1<0$, then 
\begin{align}
\boldsymbol{x}^{\top} F_1 \boldsymbol{x}+2 \boldsymbol{g_1}^{\top} \boldsymbol{x}+h_1 \leq 0 \Longrightarrow \boldsymbol{x}^{\top} F_2 \boldsymbol{x}+2 \boldsymbol{g_2}^{\top} \boldsymbol{x}+h_2 \leq 0,\label{eq29} \\
\iff  \  \exists \lambda \geq 0, \  
\lambda \begin{bmatrix} F_1 & \boldsymbol{g_1} \\
\boldsymbol{g_1}^{\top} & h_1 \end{bmatrix}
- \begin{bmatrix} F_2 & \boldsymbol{g_2} \\
\boldsymbol{g_2}^{\top} & h_2 \end{bmatrix} \succeq 0. \label{eq30} \ \ \ \ \ \ \ 
\end{align}
\end{lemma}

\begin{prop}\label{prop7}
Let set $\mathcal B_m$ be denoted using ellipsoid $\mathcal E(E_1, \boldsymbol e_1, \boldsymbol s)$ as \eqref{eq4} and let set $ \mathcal W_l$ be denoted using ellipsoid $\mathcal E(E_2, \boldsymbol e_2, \boldsymbol s)$ as \eqref{eq4}. Then, $\mathcal B_m \subset \mathcal W_l \iff$ 
\begin{equation}\label{eq31}
\begin{aligned}
\exists \Upsilon \in  \mathbb{R}^{(N+1) \times (N+1)}, \lambda \geq 0: \ \ \ \ \ \ \ \ \ \ \ \ \ \ \ \ \ \ \ \ \ \ \ \ \ \ \ \ \\
 \begin{bmatrix} 
 \lambda E_1- E_2 & E_2^{\top}\boldsymbol{e}_{2}-\lambda E_1^{\top}\boldsymbol{e}_{1} \\
\boldsymbol{e}_{2}^{\top}E_2- \lambda \boldsymbol{e}_{1}^{\top}E_1 & \boldsymbol{e}_{1}^{\top} \lambda E_1 \boldsymbol{e}_{1}-\boldsymbol{e}_{2}^{\top} E_2 \boldsymbol{e}_{2}-\lambda +1
\end{bmatrix}\\
=\Upsilon \Upsilon^{\top}, \Upsilon(i,i) \geq 0, \  i= 1\cdots N+1.
\end{aligned}
\end{equation}
\end{prop}

\begin{proof}
In representations \eqref{eq6} and \eqref{eq4}, $E_1$ and $E_2$ are positive definite, symmetric matrices, i.e., $E_1 = E_1^{\top}$, $E_2 = E_2^{\top}$. Set $\mathcal B_m$ is a convex set with nonempty interior, so $\left(\boldsymbol{s}-\boldsymbol{e}_{1}\right)^{\top} E_1\left(\boldsymbol{s}-\boldsymbol{e}_{1}\right) < 1$ holds. Thus, constraint $\mathcal B_m \subset \mathcal W_l$ is identical to stating that $ \forall \boldsymbol{x} \in \mathbb R^N, \boldsymbol{x}^{\top} E_1 \boldsymbol{x} - 2 \boldsymbol{e}_{1}^{\top} E_1 \boldsymbol{x}+ \boldsymbol{e}_{1}^{\top} E_1 \boldsymbol{e}_{1} -1 \leq 0 \Longrightarrow \boldsymbol{x}^{\top} E_2 \boldsymbol{x} - 2 \boldsymbol{e}_{2}^{\top} E_2 \boldsymbol{x}+ \boldsymbol{e}_{2}^{\top} E_2 \boldsymbol{e}_{2} -1 \leq 0 $. Based on Lemma~\ref{lemma1}, $\exists \lambda \geq 0$, for which the matrix 
\[\begin{bmatrix} 
 \lambda E_1- E_2 & E_2^{\top}\boldsymbol{e}_{2}-\lambda E_1^{\top}\boldsymbol{e}_{1} \\
\boldsymbol{e}_{2}^{\top}E_2- \lambda \boldsymbol{e}_{1}^{\top}E_1 & \boldsymbol{e}_{1}^{\top} \lambda E_1 \boldsymbol{e}_{1}-\boldsymbol{e}_{2}^{\top} E_2 \boldsymbol{e}_{2}-\lambda +1
\end{bmatrix},\] 
denoted as symbol $G$ for clarity, is a positive semidefinite matrix, i.e., $G \succeq 0$. 
Next, we prove $G \succeq 0 \iff \exists \Upsilon \in \mathbb{R}^{(N+1) \times (N+1)}$, such that $G = \Upsilon \Upsilon^{\top}, \Upsilon(i,i) \geq 0, i= 1\cdots N+1$. 

According to the Cholesky theory \cite{ref43}, $G \succeq 0$ has a decomposition of the form $G = L L^{\top}$ where $L \in \mathbb{R}^{(N+1) \times (N+1)}$ is a lower triangular matrix with nonnegative diagonal entries, so $L$ is a special case of $\Upsilon$. This proves the left-to-right implication. Based on the definition of positive semidefinite matrix, $\forall \boldsymbol{x} \in \mathbb R ^{N}, \boldsymbol{x}^{\top} \Upsilon \Upsilon^{\top} \boldsymbol{x}=\Vert \boldsymbol{x}^{\top} \Upsilon \Vert^{2} \geq 0$, let $G = \Upsilon \Upsilon^{\top}$, thus $G \succeq 0 $, which completes the right-to-left statement.
\end{proof}

\subsection{An ellipsoid contained in a polytope}\label{sec5.3}
\begin{prop}\label{prop8}
Let set $\mathcal B_m$ be denoted using ellipsoid $\mathcal E(E, \boldsymbol e, \boldsymbol s)$ as \eqref{eq4}. Let set $ \mathcal W_l$ be denoted using polytope $\mathcal P(P, \boldsymbol p, \boldsymbol s)$ as \eqref{eq3}. Then, $\mathcal B_m \subset \mathcal W_l \iff$
\begin{equation}\label{eq32}
\begin{aligned}
\sqrt{\boldsymbol{p}_{i,P}^{\top} E^{-1} \boldsymbol{p}_{i,P}}+\boldsymbol{p}_{i,P}^{\top} \boldsymbol{e} \leq p_{i}, \  i= 1\cdots N_P.
\end{aligned}
\end{equation}
\end{prop}
\begin{proof}
Let $\mathcal H^{-}(\boldsymbol{\lambda}, \mu) $ be defined as stated earlier, i.e., $\mathcal H^{-}(\boldsymbol{\lambda}, \mu)=\left\{\boldsymbol{s} \in \mathbb{R}^N \mid \boldsymbol{\lambda}^\top \boldsymbol{s} \leq \mu\right\}$. From \eqref{eq3}, polytope $\mathcal W_l$ is represented as $\left\{\boldsymbol{s} \in \mathbb{R}^N \mid P \boldsymbol{s} \leq \boldsymbol{p}\right\}$, i.e., the intersection of $N_P$ half-spaces, denoted as $\mathcal W_l= \mathcal H^{-}(\boldsymbol{p}_{1,P}, p_{1}) \cap \mathcal H^{-}(\boldsymbol{p}_{2,P}, p_{2})\cdots \cap \mathcal H^{-}(\boldsymbol{p}_{N_P,P}, p_{N_P})$.  Thus, constraint $\mathcal B_m \subset \mathcal W_l$ is identical to stating 
$\mathcal B_m \subset \mathcal H^{-}(\boldsymbol{p}_{1,P}, p_{1})$, $\mathcal B_m \subset \mathcal H^{-}(\boldsymbol{p}_{2,P}, p_{2})$, $ \cdots, \mathcal B_m \subset \mathcal H^{-}(\boldsymbol{p}_{N_P,P}, p_{N_P})$ hold. Then, based on Proposition~\ref{prop5}, $\mathcal B_m \subset \mathcal W_l$ can be formulated by checking collisions between $N_P$ half-spaces and $\mathcal B_m$.    
\end{proof}

\section{Analytical evaluation of the collision avoidance formulations}\label{sec6}
To illustrate the advantages of the proposed collision avoidance formulations in efficiency, a few prominent methods in literature to model collision constraints (\ref{eq2f}) are used as a benchmark to compare. We highlight the advantages of saving computational resources in aspects of problem size and setting initial guesses.
\subsection{Problem size}\label{sec6.1}
When considering the differentiable formulations of separating constraints (\ref{eq2f}), i.e., $ \mathcal B_m \cap \mathcal O_n=\emptyset $, between two convex sets, one prominent way is reformulating the distance between them using dual techniques, such as the proposed collision formulations by Zhang et al. \cite{ref3}, where strong duality theory of convex optimization is utilized to achieve exact collision avoidance. Additionally, Helling et al. \cite{ref6} leverage Farkas' Lemma to arrive at conditions ensuring collision avoidance between two polytopes. However, these methods require introducing many dual variables to be optimized in the NLP formulation. The resulting growth in problem size can be computationally expensive. On the contrary, our formulations using the hyperplane separation theorem need fewer variables, admitting an improved computational performance compared to existing approaches.

Table~\ref{tab1} summarizes the needed number of auxiliary variables to formulate $ \mathcal B_m \cap \mathcal O_n=\emptyset $ when they are polyhedral or ellipsoidal sets. It can be seen that generally the proposed formulations need $N+1$ variables, e.g., $2+1$ variables in two dimensions and $3+1$ variables in three dimensions,  to construct the separating hyperplane between $\mathcal B_m$ and $ \mathcal O_n$, not depending on the complexity of set $\mathcal B_m $ or $\mathcal O_n$. As a contrast, the existing formulations in \cite{ref3, ref6} identify additional variables proportional to the complexity of $\mathcal B_m$ and $ \mathcal O_n$ being represented, always needing more variables than the proposed formulations. It also states that in the presence of polytopes when the number of edges increases, more variables are needed to formulate using formulations in \cite{ref3, ref6}, while the number of needed variables remains unchanged using the proposed formulations. Typically, obstacles are various nonconvex polytopes composed of many convex polytopes, so this advantage of the proposed formulations is highly beneficial for efficient collision avoidance in real-world scenarios.

\begin{table}[t]
  \centering
  \caption{The number of additional variables based on different formulations}
  \begin{threeparttable}
    \begin{tabular}{cccc}
    \toprule
    \multirow{2}[2]{*}{Methods} & Hyperplane & Dual & Farkas' \\
          &    separation theorem   &   formulations    & Lemma  \\
    \midrule
    Two polytopes  & $N+1$   & $N_{m}+N_{n}$ \tnote{1} & $N_{m}+N_{n}$ \\
    Two ellipsoids & $N+1$     & $2(N+1)$    &  - \tnote{2} \\
    Polytope, ellipsoid & $N+1$   & $N_{m}+N$  & - \\
    \bottomrule
    \end{tabular}%
        \begin{tablenotes}   
        \footnotesize 
        \item[1] The least number of faces to construct a bounded polytope in N dimension is $N+1$, i.e., $N_{m}\geq N+1, N_{n}\geq N+1$, indicating the number using Dual formulations is always larger than that using Hyperplane separation theorem in each row.
        \item[2] - means Farkas' Lemma cannot achieve it.
      \end{tablenotes}
    \end{threeparttable}
  \label{tab1}%
\end{table}%

\subsection{Geometry-based initial guess}\label{sec6.2}
Since the planning OCP \eqref{eq2} is a non-convex optimization problem,  good initialization in state variables, control variables, and auxiliary variables used to formulate collision avoidance constraints can improve the computational performance and the resulting solution quality. Much existing work has designed efficient initial guesses in state variables and control variables, e.g., Hybrid $\tx{A}^{\star}$ algorithm \cite{ref1, ref22}. Some researchers choose to initialize auxiliary variables by solving a convex optimization problem \cite{ref22, ref32}. However, it is noted that the number of auxiliary variables is very large, especially needed using the aforementioned dual formulations and Farkas' Lemma, so this numerical optimization way is not efficient. Other researchers \cite{ref1, ref33} prefer to initialize auxiliary variables with constants, e.g., zero or a small value to avoid a singularity, but this way may affect solution quality on account of not being problem-dependent.

Owing to the proposed formulations using the hyperplane separation theorem, we propose a geometry-based method to initialize auxiliary variables. According to Theorem~\ref{theorem1}, when convex sets $\mathcal B_m$ and $\mathcal O_n$ are disjoint, i.e.,  $\mathcal B_m \cap \mathcal O_n =\emptyset$, there exists a separating hyperplane, which is described by auxiliary variables, e.g., $\boldsymbol{\lambda}, \mu$ in formulation (\ref{eq13}). Rather than initializing these auxiliary variables, as an alternative, we initialize the separating hyperplane from a view of geometry. It is done by finding a middle hyperplane between $\mathcal B_m $ and $ \mathcal O_n$, i.e., let auxiliary variables be initialized with a coefficient vector describing the separating hyperplane which is orthogonal to the line crossing the sample on the initial path and the centroid\footnote{The centroid of a polytope is the arithmetic average of the coordinates of its vertices and the centroid of an ellipsoid is its center.} of an obstacle, and crosses certain point\footnote{Let $(x_1,y_1)$ and $(x_2,y_2)$ represent the coordinates of the sample and the centroid. The position of a certain point is denoted as $\gamma(x_1,y_1)+(1-\gamma)(x_2,y_2) $, where $\gamma$ is a problem-dependent weight factor. } between the sample and the centroid. An example of this initial separating hyperplane is illustrated in Fig~\ref{fig1}. Because the initial guesses in auxiliary variables using the proposed geometry-based method can be obtained by calculating analytical solutions directly, it is simple and computationally efficient.

In the previous paragraph, we introduce a general method of how to initialize the separating hyperplanes, i.e., using an orthogonal hyperplane, but the geometry-based method is not limited to this. Because the initialization of the separating hyperplanes is highly problem-dependent, the initialization method should be carefully selected to optimize computational performance in various planning cases.

\begin{figure}[!t]
\centering
\subfloat[$\gamma=0.5$]{\includegraphics[width=0.35\textwidth]{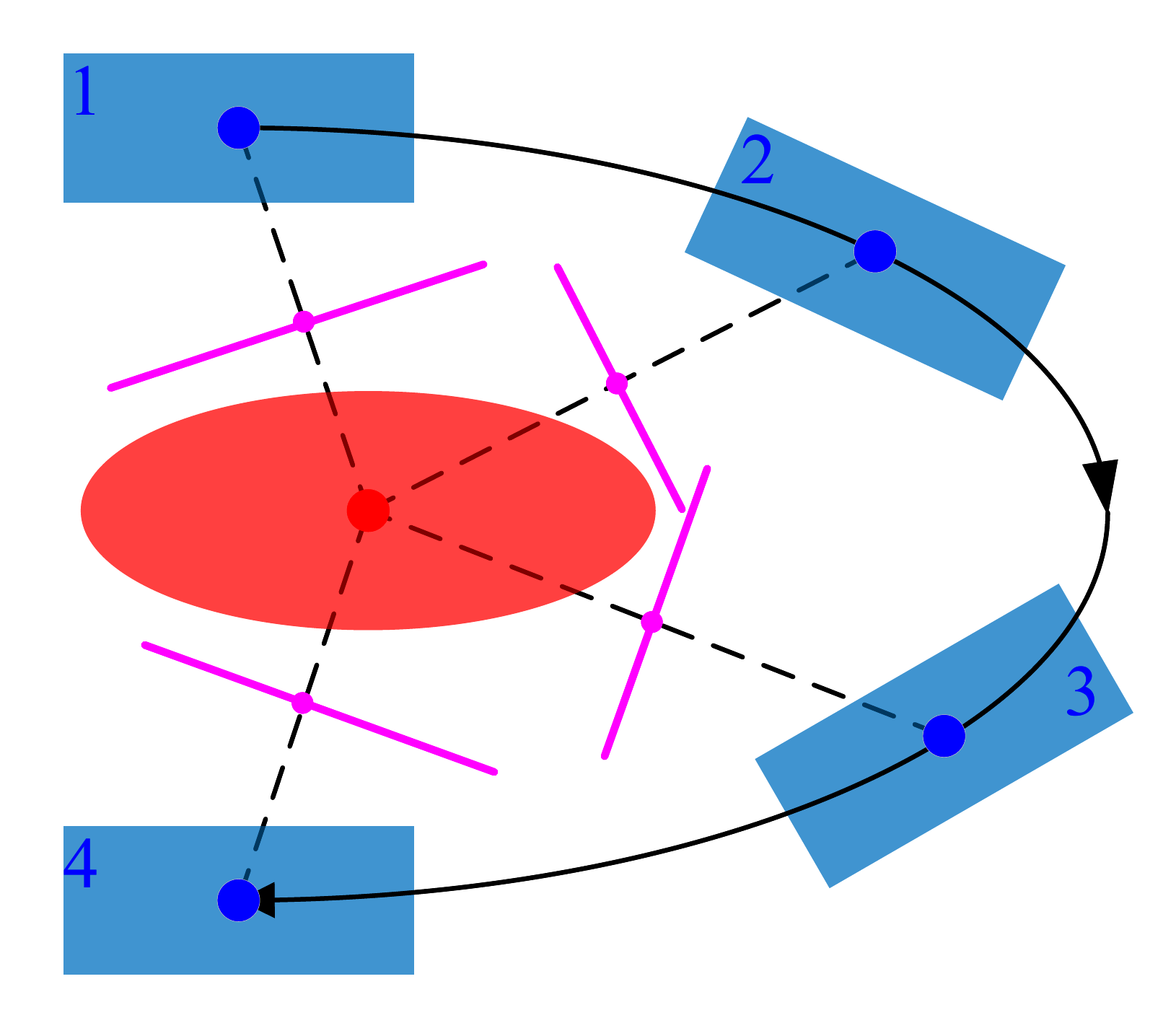}%
\label{fig1a}}
\hfill
\subfloat[$\gamma=0.75$]{\includegraphics[width=0.4\textwidth]{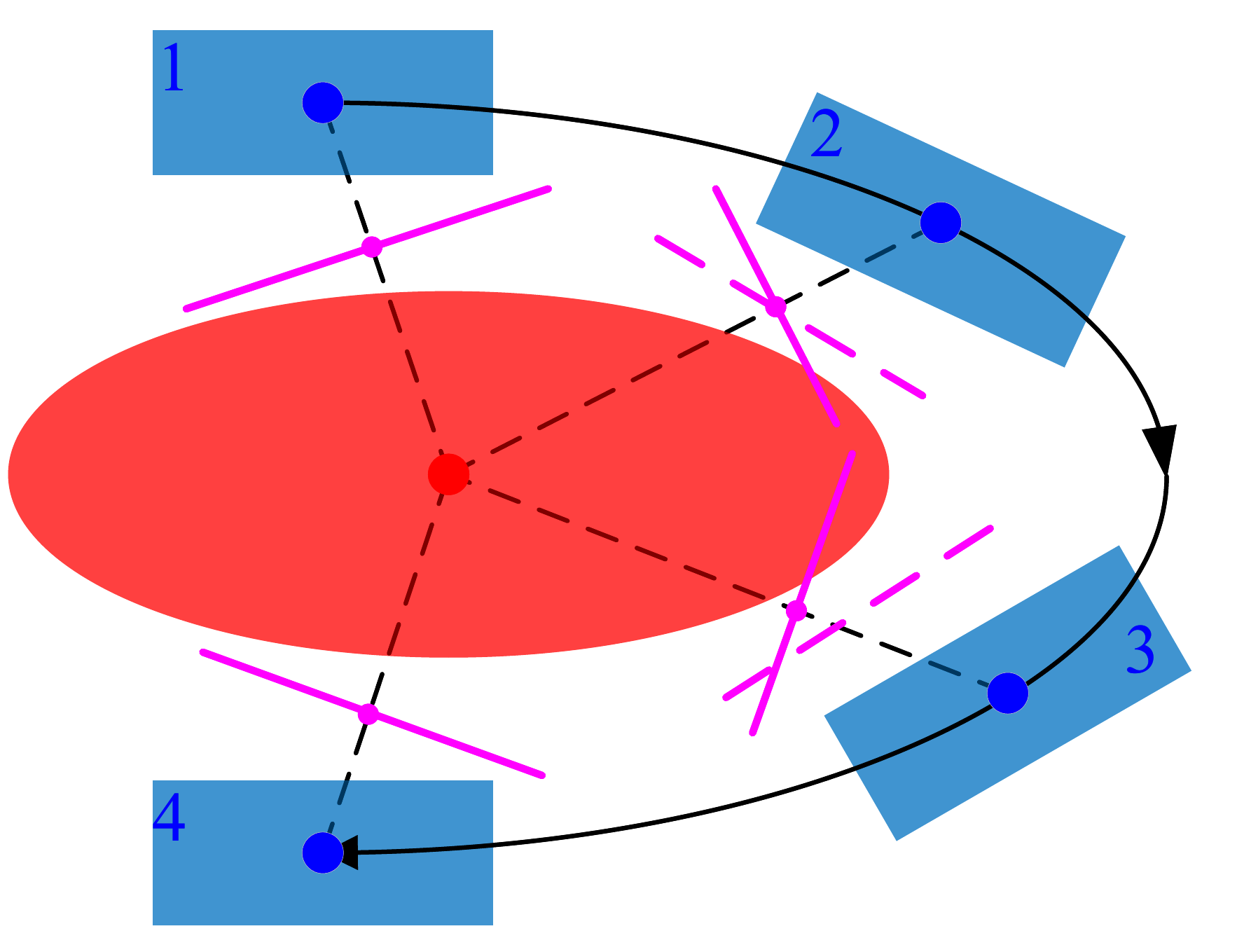}%
\label{fig1b}}
\caption{Positions of the proposed initial separating hyperplanes when weight factor $\gamma$ is set as 0.5 and 0.75, depending on the size of the obstacle. The controlled object $\mathcal B$ is denoted as a polygon, depicted in blue. The obstacle $\mathcal O$ is denoted as an ellipse, depicted in red. Centers of $\mathcal B$ and $\mathcal O$ are depicted in blue point and red point, respectively. The initial path is depicted in black solid line. The magenta lines in samples 1--4 are the initial separating hyperplanes. In (a), the size of $\mathcal O$ is smaller, we choose $\gamma=0.5$ to keep the hyperplane separating $\mathcal B$ from $\mathcal O$ in each sample, while in (b), $\gamma=0.75$ is set to keep the separating hyperplanes as far away from $\mathcal O$ as possible. Although the hyperplanes in samples 2 and 4 are not fully separating $\mathcal B$ from $\mathcal O$, such initialization may still help the solver quickly find the resulting separating hyperplanes in the neighborhood, i.e., starting from the solid lines, the solver may more easily find the magenta dotted lines.} 
\label{fig1}
\end{figure}

\section{Applications}\label{sec7}
In this section, we show two applications of the proposed unified planning framework involving containing constraints and separating constraints. One case is autonomous parking in tractor-trailer vehicles in a limited moveable environment. The other case is car overtaking within narrow curved lanes. Additionally, in each case, comparisons with the state of the art on the same premise are done to illustrate the advantages of the proposed formulations.
\subsection{Autonomous Parking for tractor-trailer vehicles}\label{sec7.1}

\subsubsection{Problem description}
Tractor-trailer vehicles are longer than standard cars, resulting in larger turning curvatures and a wider swept area, which makes their control difficult. The difficulty of parking is also significantly influenced by the parking geometry, whether it be parallel, angular, or perpendicular. Besides, in a limited nonconvex parking region, maintaining an acceptable safety distance from obstacles to prevent collisions poses a demanding task even for skilled drivers\cite{ref7, ref49}, so it is an intractable issue for tractor-trailer vehicles parking within confined surroundings. Since the execution of reverse parking presents greater complexity in typical scenarios, autonomous reverse parking in tractor-trailer vehicles is studied by exploiting separating constraints \eqref{eq12} and containing constraints \eqref{eq28}.

\subsubsection{Case description}
In the example of Fig.~\ref{fig3}(a) and Fig.~\ref{fig3}(b), the driving environment $\mathcal W$ is modeled by a convex polygon, depicted in green, denoted as $\mathcal W:= \left\{\boldsymbol{s} \in \mathbb{R}^2 \mid P_1 \boldsymbol{s} \leq \boldsymbol{p_1}\right\},$ where $P_1 =[0,1,0,-1;1,0,-1,0]^\top,\boldsymbol{p_1}=[10,10,10,6]^\top$. The obstacle $\mathcal O$ is modeled by a convex polygon, depicted in red solid. Let $V_O \in\mathbb R^{2\times 4}$ represent a matrix of vertices in $\mathcal O$, denoted as $V_O=\begin{bmatrix} 7, 7, -6, -6; -3, -10, -10, -3\end{bmatrix}$. The tractor-trailer vehicle $\mathcal B$ is modeled using a union of two convex polygons $\mathcal B_1$ and $\mathcal B_2$, depicted in magenta and blue. Polygons $\mathcal B_1$ and $\mathcal B_2$ are connected at the common articulated joint point. Because the position and orientation of the vehicle change, its set  $\mathcal B_i(\boldsymbol{\xi})$ should be written to depend on the vehicle state $\boldsymbol{\xi}$, $i=1,2$. The matrix of vertices in $\mathcal B_i(\boldsymbol{\xi})$ is represented by $V_i(\boldsymbol{\xi})\in\mathbb R^{2\times 4}$. Matrix $V_i(\boldsymbol{\xi})$ can be obtained by the rotation and translation of the initial matrix as
\begin{equation}\label{eq42}
\begin{aligned}
V_i(\boldsymbol{\xi}) =R_i(\boldsymbol{\xi})V_i(0)+\boldsymbol{T_i}(\boldsymbol{\xi})\boldsymbol{1}^\top,
\end{aligned}
\end{equation}
where $V_i(0)$ is denoted as the initial matrix, denoted as $V_1(0) =\begin{bmatrix} 1.5, 1.5,-0.5,-0.5; 1,-1,-1, 1\end{bmatrix}$, $V_2(0) =\begin{bmatrix}  0.5, 0.5, -5, -5;
1, -1, -1, 1\end{bmatrix}$. The rotation matrix is represented as $R_i(\boldsymbol{\xi})$. The translation vector is $\boldsymbol{T_i}(\boldsymbol{\xi})$.
The collision avoidance requirements, i.e.,  $\mathcal B \cap \mathcal O =\emptyset $ and $\mathcal B \subset \mathcal W$ are approached by using the proposed \eqref{eq12} and \eqref{eq28}, respectively.

\begin{figure*}[!t]
\centering
\subfloat[Initial path]{\includegraphics[width=0.3\textwidth]{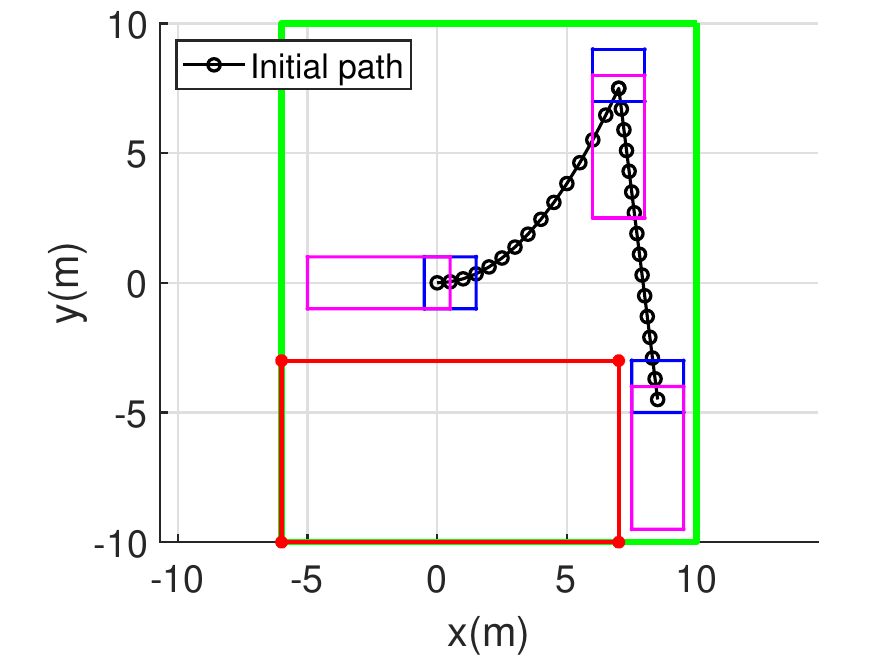}%
\label{fig3a}}
\subfloat[Parking a tractor-trailer vehicle]{\includegraphics[width=0.3\textwidth]{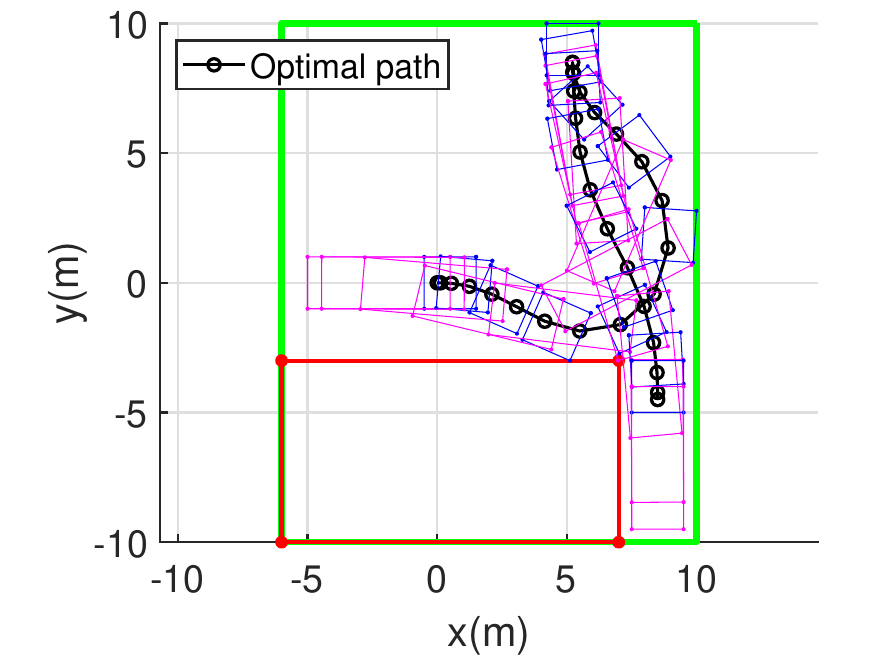}%
\label{fig3b}}
\subfloat[Parking under one obstacle]{\includegraphics[width=0.3\textwidth]{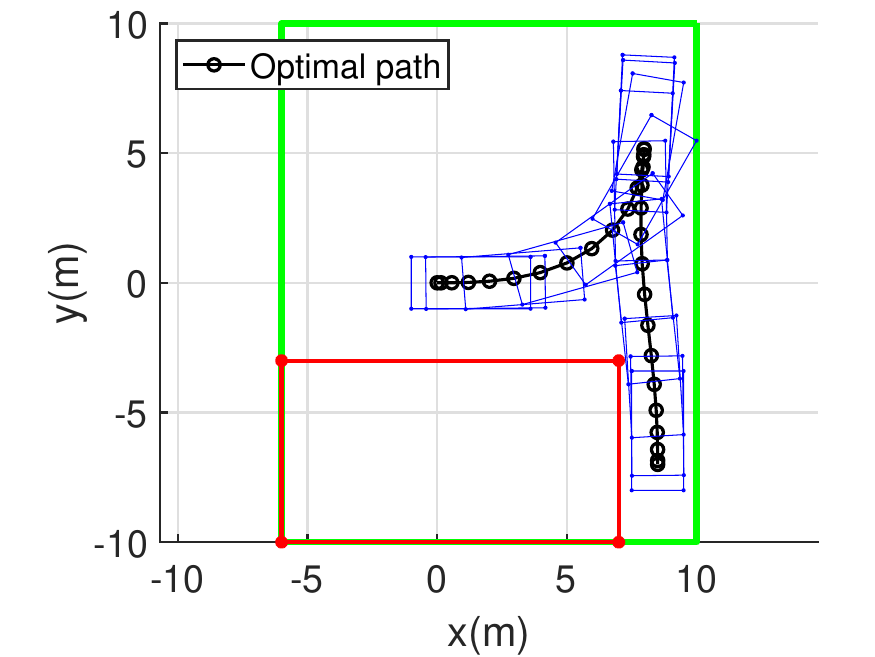}%
\label{fig3c}}
\hfill
\subfloat[Parking under two obstacles]{\includegraphics[width=0.3\textwidth]{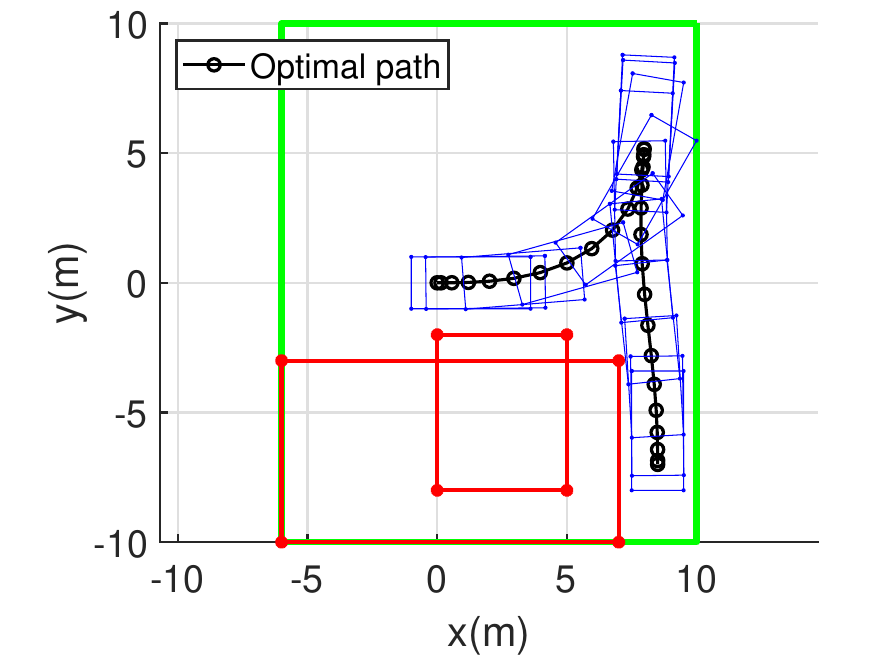}%
\label{fig3d}}
\subfloat[Parking under three obstacles]{\includegraphics[width=0.3\textwidth]{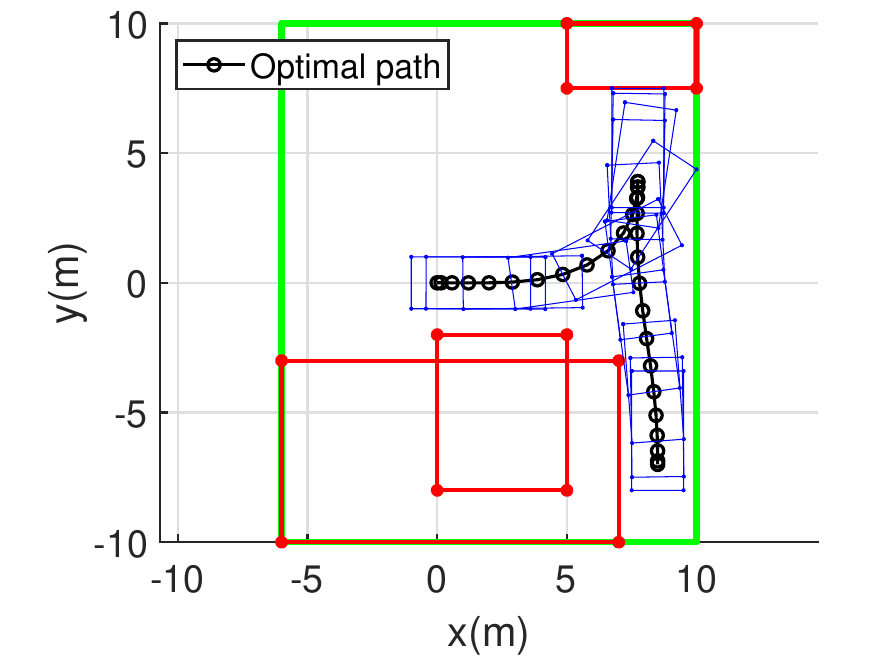}%
\label{fig3e}}
\subfloat[Parking under four obstacles]{\includegraphics[width=0.3\textwidth]{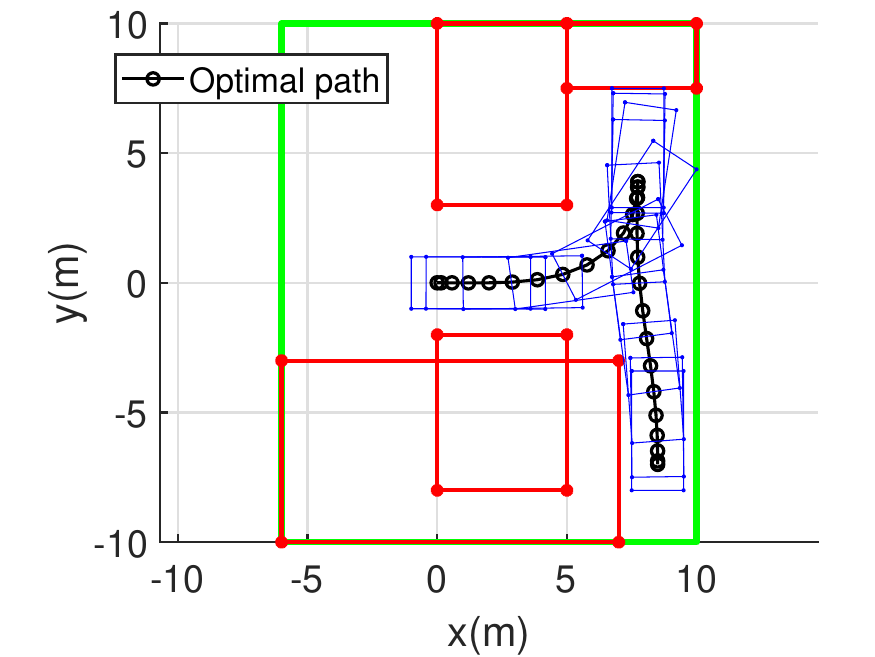}%
\label{fig3f}}
\caption{Subplot (a) illustrates the initial path using piecewise polynomials. The interim position used to obtain an initial path is denoted as $[\SI{7}{m},\SI{7.5}{m}]$. Subplot (b) shows the planned optimal path of a tractor-trailer vehicle.  In (b), the starting state $\boldsymbol{\xi_{\tx{init}}}$  is  $\begin{bmatrix}\SI{0}{m}, \SI{0}{m}, \SI{0}{\degree}, \SI{0}{\degree}, \SI{0}{m/s}, \SI{0}{\degree}\end{bmatrix}$. The ending state $\boldsymbol{\xi_{\tx{final}}}$ is $\begin{bmatrix}\SI{8.5}{m}, \SI{-4.5}{m},\SI{90}{\degree}, \SI{90}{\degree}, \SI{0}{m/s},\SI{0}{\degree}\end{bmatrix}$.  Subplots (c)--(f) are the optimal paths of a single car under different obstacles. In (c)--(f), the starting state $\boldsymbol{\xi_{\tx{init}}}$  is  $\begin{bmatrix}\SI{0}{m}, \SI{0}{m}, \SI{0}{\degree}, \SI{0}{m/s}, \SI{0}{\degree}\end{bmatrix}$. The ending state $\boldsymbol{\xi_{\tx{final}}}$ is $\begin{bmatrix}\SI{8.5}{m}, \SI{-7}{m},\SI{90}{\degree},\SI{0}{m/s},\SI{0}{\degree}\end{bmatrix}$.} 
\label{fig3}
\end{figure*}

\subsubsection{Vehicle model and objective function} \label{sec7.1.2}
The vehicle motion is described using a kinematic model, extended with a trailer \cite{ref50,ref51} as
\begin{equation}\label{eq37}
\begin{array}{l}
    \dot{x}=v \cos (\theta_1),\\
    \dot{y}=v \sin (\theta_1),\\
    \dot{\theta_1}=v\tan(\delta)/L_1,\\
    \dot{\theta_2}=v\sin(\theta_1-\theta_2)/L_2,\\
    \dot{v}=a,\\
    \dot{\delta}=\omega,
\end{array}
\end{equation}
where $L_1=\SI{1}{m}$ and $L_2=\SI{4.5}{m}$ are wheelbases of the tractor and the trailer, respectively. The position of the tractor-trailer joint is identical to the vector $\left(x,y\right)$, and $\theta_1$, $\theta_2$ represent yaw angles of the tractor and the trailer with respect to the horizontal axle. The velocity and acceleration of the tractor-trailer joint are $v$ and $a$. The steering angle and the gradient of the steering angle of the tractor are denoted by $\delta$ and $\omega$, respectively. The state and control vectors can be summarized as
\begin{equation}\label{eq38}
\begin{array}{c}
\boldsymbol{\xi} = \begin{bmatrix}x & y & \theta_1 & \theta_2 & v &\delta \end{bmatrix}^{\top},
\boldsymbol{u} = \begin{bmatrix}a & \omega \end{bmatrix}^{\top},  \\
\end{array}
\end{equation}
so the vehicle dynamics can be denoted as $\dot{\boldsymbol{\xi}} = f\left(\boldsymbol{\xi},\boldsymbol{u}\right)$. The feasible values of $\theta_1,\theta_2,v,\delta$ are limited to $|\theta_1| \leq \SI{180}{\degree}$, $|\theta_2| \leq \SI{180}{\degree}$, $ |v| \leq \frac{5}{3.6} \si{m/s^2}$, $|\delta| \leq 40 \si{\degree}$. The feasible inputs are given by $|a|\leq \SI{1}{m/s^2}$ and $|\omega|\leq \SI{5}{\degree/s^2}$. Specifically, the articulated joint angle at the tractor-trailer joint is limited using the constraint
\begin{equation}\label{eq39}
\vert\theta_1-\theta_2\vert \leq 60 \si{\degree}. 
\end{equation}
In parking cases, we expect the autonomous vehicle to complete the parking maneuvers subjected to minimum traveling time. However, the parking duration, denoted as $t_f$, is not known in advance. To address this problem, We introduce a new independent variable $\tau \in [0, 1]$, and express the time as $t=t_f \tau$, where $t_f\geq 0$ is a scalar optimization variable that needs to be added to the problem. Thus, the vehicle dynamics in \eqref{eq37} become a function of $\tau$ as
\begin{equation}\label{eq40}
\frac{\tx{d}\boldsymbol{\xi}(\tau)}{\tx{d}\tau} = t_f f(\boldsymbol{\xi}(\tau), \boldsymbol{u}(\tau)).
\end{equation}

Apart from the penalty on traveling time, the control input $a$ and $\omega$ should be small to yield smooth trajectories. Thus, a typical time-energy cost function in OCP \eqref{eq2} is formulated as 
\begin{equation}\label{eq41}
\begin{aligned}
 \underset {\boldsymbol{\xi}, \boldsymbol{u}, t_f,(\cdot)}{\text{min}} t_f \left(r + \int_{0}^{1} \left\Vert \boldsymbol{u}\left(\tau\right)\right\Vert_{Q}^{2} \tx{d}\tau\right),
\end{aligned}
\end{equation}
where $r$ is a penalty factor that trades the parking time with the rest of the cost. The notation \(\left\Vert \cdot\right\Vert_{Q}^{2}\) represents squared Euclidean norm weighted by matrix \(Q\). For the studied cases we choose $r=1, Q={\rm diag}(100, 200)$.

The OCP (\ref{eq2}) is generally reformulated as a discrete NLP problem that can be solved using off-the-shelf solvers. The whole region $[0,1]$ is discretized into \( {k}_{f}+1\) parts, i.e., \( 1 = {k}_{f} \Delta \tau\), where $\Delta \tau$ is the sampling interval. Variable $\tau$ is replaced with \( \tau = {k} \Delta \tau\), where ${k}=0$, $\ldots$, \( {k}_{f}\). Thus, the OCP (\ref{eq2}) involving collision avoidance can be formulated as a discrete NLP problem
{\allowdisplaybreaks
\begin{subequations}\label{eq36}
\begin{align}
\underset {\boldsymbol{\xi}, \boldsymbol{u},t_f,(\cdot)}{\text{min}} & \ r t_f + \sum_{k=0}^{k_f-1} \tilde{\ell}(\boldsymbol{u}(k),t_f)\label{eq36a}\\
\text {s.t.}  \ \ & \boldsymbol{\xi}(0)=\boldsymbol{\xi_{\tx{init}}},\ \boldsymbol{\xi}(k_\tx{f})=\boldsymbol{\xi_{\tx{final}}}, \label{eq36b}\\
 & \boldsymbol{\xi}(k+1) =\tilde{f}\left(\boldsymbol{\xi}(k),\boldsymbol{u}(k),t_f\right), \label{eq36c}\\
 & \boldsymbol{\xi}(k) \in \mathcal{X},\ \boldsymbol{u}(k) \in \mathcal{U}, \label{eq36d}\\
 & (\ref{eq12}), (\ref{eq28}), \label{eq36e}
\end{align}
\end{subequations}}%
where symbol \( (\cdot )\) is a shorthand notation for decision variables that need to be added to the vector of existing optimization variables, e.g., $\mu$ in formulation (\ref{eq12}). Symbol \(\tilde{f}\) denotes the discretized model dynamics of \eqref{eq40} and symbol $\tilde{\ell}$ is the discretized form of the state cost $t_f\left\Vert \boldsymbol{u}\left(\tau\right)\right\Vert_{Q}^{2}$ in function \eqref{eq41}. In the simulation, the 4-th order Runge-Kutta method \cite{ref33} is chosen to express the discretized $\tilde{\ell}$ and \(\tilde{f}\).

Since the resulting NLP (\ref{eq36}) is nonconvex, multiple local solutions exist depending on different initial guesses. An initial guess in the neighborhood of a local solution is more likely to cause the NLP solver to return that local solution. In a reverse parking scenario, we use piecewise polynomials to initialize states $x,y$. In the example of Fig~\ref{fig3}(a) and Fig~\ref{fig3}(b), we expect that the vehicle parks into the garage by reversing only once, so the initial guess in positions is defined by using a second-order polynomial from the starting position to an interim position and a linear function from the interim position and ending position (see Fig.~\ref{fig3}(a)). The yaw angle is initialized as the angle between a line tangent to the polynomial trajectory at each sample and the horizontal axle. Other states and inputs are initialized as zero. As mentioned in Section~\ref{sec6.2}, let the initial separating hyperplane, which is used to initialize auxiliary variables $\boldsymbol{\lambda}$, $\mu$ in formulation \eqref{eq12}, be initialized with the hyperplane orthogonal to the line crossing the sample on the polynomial trajectory and the centroid of an obstacle. In the example, we set $\gamma=0.75$ to position the initial hyperplane near the vehicle. Finding an initial guess of parking time $t_f$ is challenging and highly problem-dependent, so for this reverse parking case, we use an estimated value of \SI{50}{s} to warmly start $t_f$ \cite{ref3,ref22}. 

Simulations are conducted in MATLAB R2020a and executed on a laptop with AMD R7-5800H CPU at 3.20 GHz and 16GB RAM. Since we use a multiple-shooting approach, the problem is formulated as a collection of $k_\tx{f}$ phases. In this case, $k_\tx{f}=30$ phases are chosen. The problems are then implemented using CasADi \cite{ref46}. The resulting discrete NLP (\ref{eq36}) is solved using solver IPOPT \cite{ref47}.

\subsubsection{Solution}
Simulation results are visually interpreted in Fig.~\ref{fig3}(b). It can be seen that the planned optimal path is smooth and collision-free. Vertices of polygons $\mathcal B_1$ and $\mathcal B_2$ do not enter into obstacle $\mathcal O$, and vertices of $\mathcal O$ are prevented from passing the edges of $\mathcal B_1$ and $\mathcal B_2$. It can be noticed that since the parking environment $\mathcal W$ is limited, the vehicle keeps close to the environment boundaries to allow larger turning curvatures, but it is contained within $\mathcal W$ all the time. In the resulting solution, the vehicle conducts a reverse behavior to adjust its pose to enter into the confined parking place. Additionally, the implementing time to solve problem \eqref{eq36} using IPOPT is \SI{2.264}{s} and the resulting parking duration is \SI{69.89}{s}.

\begin{table}[t]
  \centering
  \caption{The result of computational cost by using different methods}
  \begin{threeparttable}
    \begin{tabular}{ccrrrr}
    \toprule
    \multirow{2}[2]{*}{Performance} & \multirow{2}[2]{*}{Algorithms} & \multicolumn{4}{c}{The number of placed obstacles} \\ 
    \cline{3-6}
    & & One & Two & Three & Four \\
    \midrule
    \multirow{3}[2]{*}{\parbox{20mm}{Number of auxiliary variables to be optimized}}  & proposed     & \textbf{306}& \textbf{396} & \textbf{486} & \textbf{576} \\
          & \cite{ref3}     & 456   & 696   & 936   & 1176 \\
          & \cite{ref6}     & 456   & 696   & 936   & 1176 \\
    \midrule
    \multirow{3}[2]{*}{Solving time (s) } & proposed     & \textbf{1.92} & \textbf{1.336} & \textbf{2.239} & \textbf{3.048} \\
          & \cite{ref3}     & 2.03  & 3.008 & 20.086 & 19.977 \\
          & \cite{ref6}     & 4.098 & 4.66  & 18.449 & 18.701 \\
    \midrule
    \multirow{3}[2]{*}{Objective value} & proposed     & 73.45 & 73.45 & 75.13 & 75.13 \\
          & \cite{ref3}     & 73.45 & 73.45 & \textcolor[rgb]{ 1,  0,  0}{77.702}\tnote{1} & 75.13 \\
          & \cite{ref6}     & 73.45 & 73.45 & 75.13 & 75.13 \\
    \midrule
    \multirow{3}[2]{*}{$t_f$ (s)} & proposed     & 52.56 & 52.56 & 51.55 & 51.55 \\
          & \cite{ref3}    & 52.56 & 52.56 & \textcolor[rgb]{ 1,  0,  0}{54.17} & 51.55 \\
          & \cite{ref6}     & 52.56 & 52.56 & 51.55 & 51.55 \\
    \bottomrule
    \end{tabular}%
        \begin{tablenotes}   
        \footnotesize       
        \item[1] The red values represent the obtained other local solutions based on the same initial guess. 
      \end{tablenotes}
  \label{tab3}%
  \end{threeparttable}
\end{table}%

\subsubsection{Comparisons in problem size} \label{sec7.1.5}
This part is used to support the advantage of improving computational performance on account of admitting a reduced problem size compared to existing approaches. It is noted that the nonconvex parking problem \eqref{eq36} may converge to multiple local solutions even if based on the same initial guess. Indeed, a vehicle may take infinitely many paths to reach the goal, especially for a tractor-trailer vehicle with more states. However, we expect to compare the computational performance of the proposed algorithms with the state of the art under the condition of converging to the same local solution, so we design a new case of a single-car parking problem. Here, we use a simplified vehicle model of \eqref{eq37}, with $L_1=2.6$. The vehicle dynamics is modelled with the reduced state vector $\boldsymbol{\xi} = \left[x,y,\theta_1,v,\delta \right]^{\top}$ and control vector
$\boldsymbol{u}\left(s\right) = \left[a,\omega\right]^{\top} $. Let the initial matrix $V_1(0)$ be denoted as $V_1(0) =\begin{bmatrix} 3.6,3.6,-1,-1;
1,-1,-1,1\end{bmatrix}$.

To keep the same premise of the NLP configuration, the initial guess in states $x,y,\theta_1$ is a linear function (starting from the beginning and ending at the final values). The rest of the states and inputs are set as zero. Let $\boldsymbol{\lambda}$ in formulation \eqref{eq12} be initialized with a positive value vector to prevent a singularity and $\mu$ be set to zero. We conduct numerical experiments by comparing formulation \eqref{eq12} with the formulations in Proposition 3 of \cite{ref3} and in Section III-A-2 of \cite{ref6}. Others remain unchanged in NLP \eqref{eq36}.  

We design four scenarios to test collision avoidance with different placement of obstacles. The resulting optimal trajectories are illustrated in Fig.~\ref{fig3}(c)--(f). The computational results are shown in Table~\ref{tab3}. It can be seen that these planned trajectories are collision-free and smooth. Apart from the scenario of three obstacles, the solver converges to the same local solution using different methods. Alternatively, we can verify the comparison results in the other three scenarios. The primal observation in Table \ref{tab3} is the least number of auxiliary variables by exploiting the proposed formulation (\ref{eq12}) in each scenario,  which illustrates that using the proposed formulation (\ref{eq12}) shows a considerable reduction in the problem size of NLP formulations compared to the other two existing formulations. From Fig.~\ref{fig3}(c) to Fig.~\ref{fig3}(f), we increase another three polygonal obstacles orderly, and it can be seen that the number of variables experiences a relatively low increase using the proposed formulation (\ref{eq12}) but a sharp increase using formulations in \cite{ref3, ref6}. In Fig.~\ref{fig3}(c)--(f), our method needs $3 \times k_f = 90$ auxiliary variables ($\boldsymbol{\lambda} \in \mathbb{R}^2, \mu \in \mathbb{R}$ in formulation \eqref{eq12}) per obstacle. On the contrary, existing formulations in \cite{ref3, ref6} identify variables associated with the edges of a vehicle and an obstacle, depending on the number of the total edges of the vehicle and all obstacles, i.e., adding $(4+4)\times k_f =240$ auxiliary variables per obstacle. The comparison of needed additional variables states that when the number of edges of each obstacle increases, more variables are needed to formulate constraints of \cite{ref3, ref6}, while the number of variables by the proposed formulation (\ref{eq12}) remains unchanged. Besides, it can be seen that as more obstacles are considered, it takes more time to solve an NLP. In each scenario, since the problem size using the proposed algorithm is smaller, solvers can converge to the same local optimums faster, which shows that the proposed method is more efficient compared with other algorithms.

\subsection{Overtaking in curved lanes}\label{sec7.2}

\begin{figure*}[!t]
\centering
\subfloat[Optimal trajectory in Scenario A]{\includegraphics[width=0.5\textwidth]{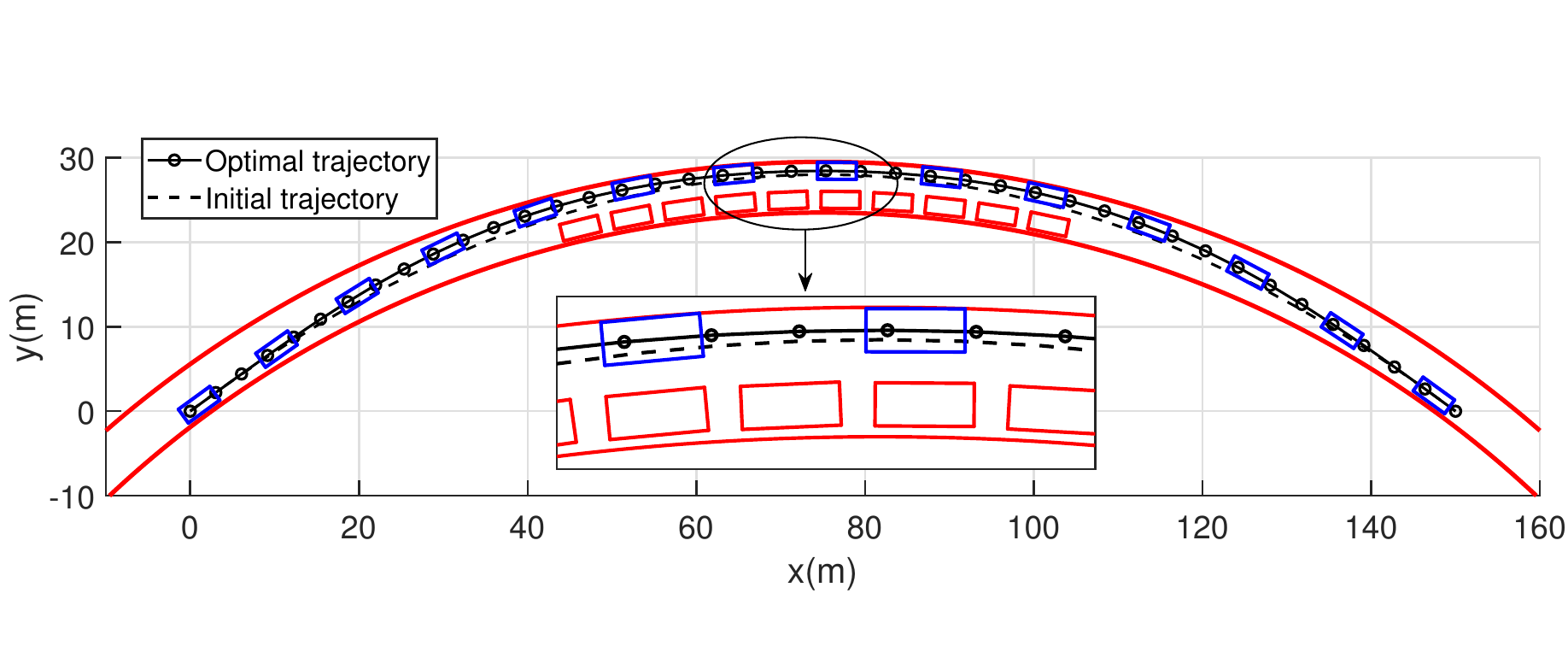}%
\label{fig2a}}
\hfil
\subfloat[The resulting states in Scenario A]{\includegraphics[width=0.5\textwidth]{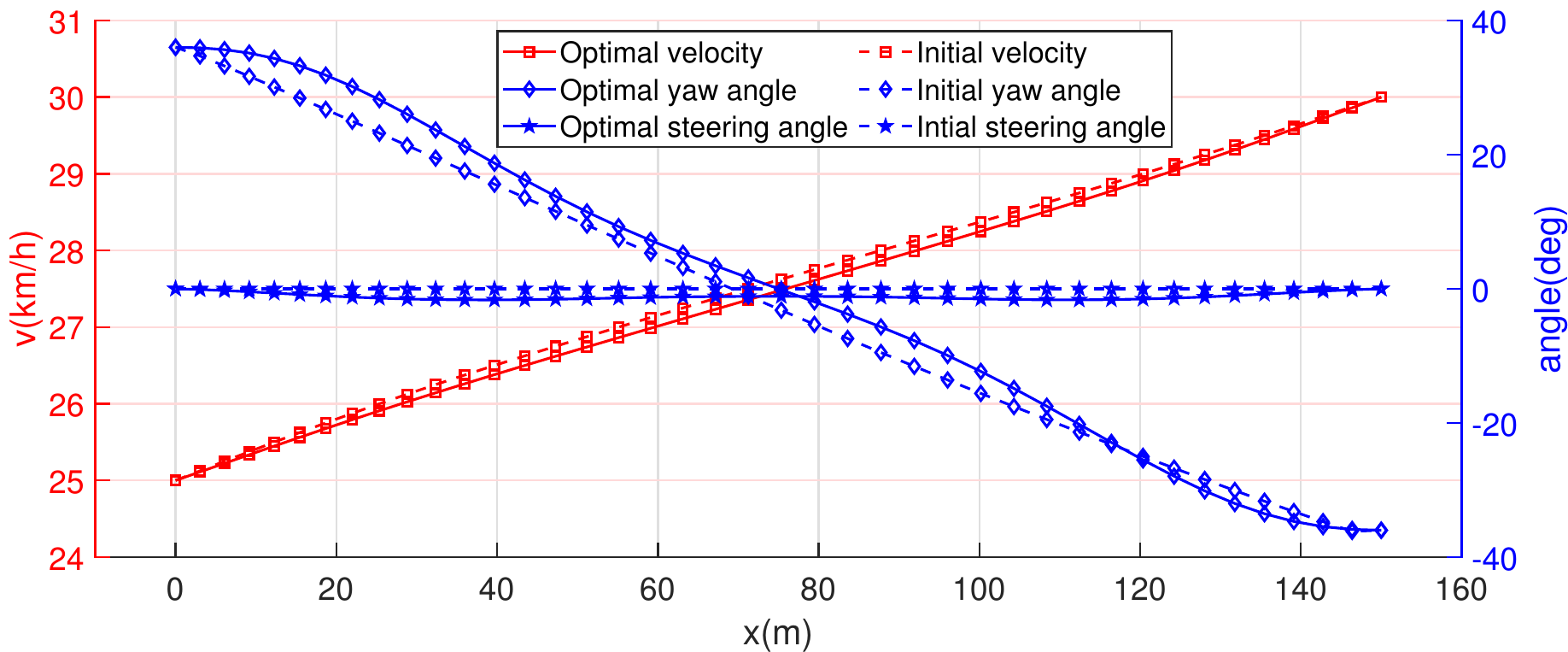}%
\label{fig2b}}
\hfil
\subfloat[Optimal trajectory in Scenario B]{\includegraphics[width=0.5\textwidth]{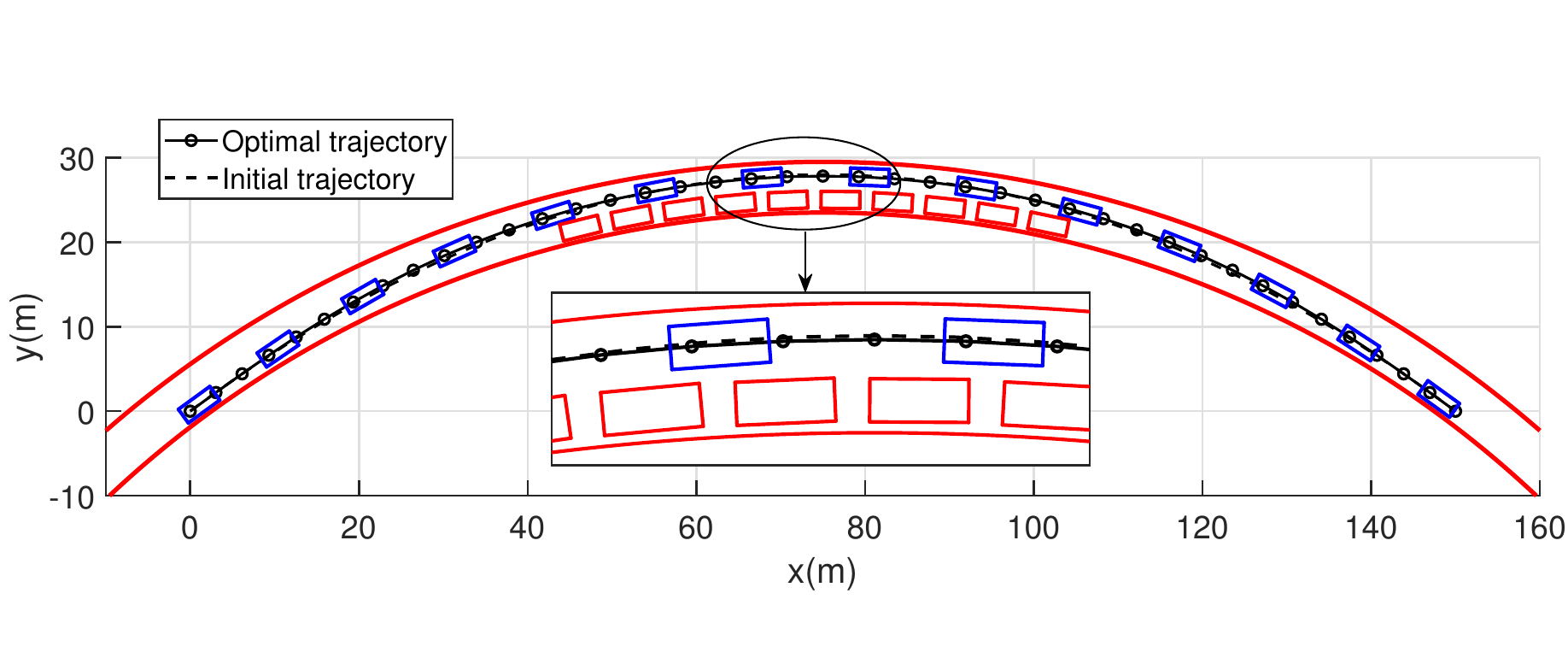}%
\label{fig2c}}
\hfil
\subfloat[The resulting states in Scenario B]{\includegraphics[width=0.5\textwidth]{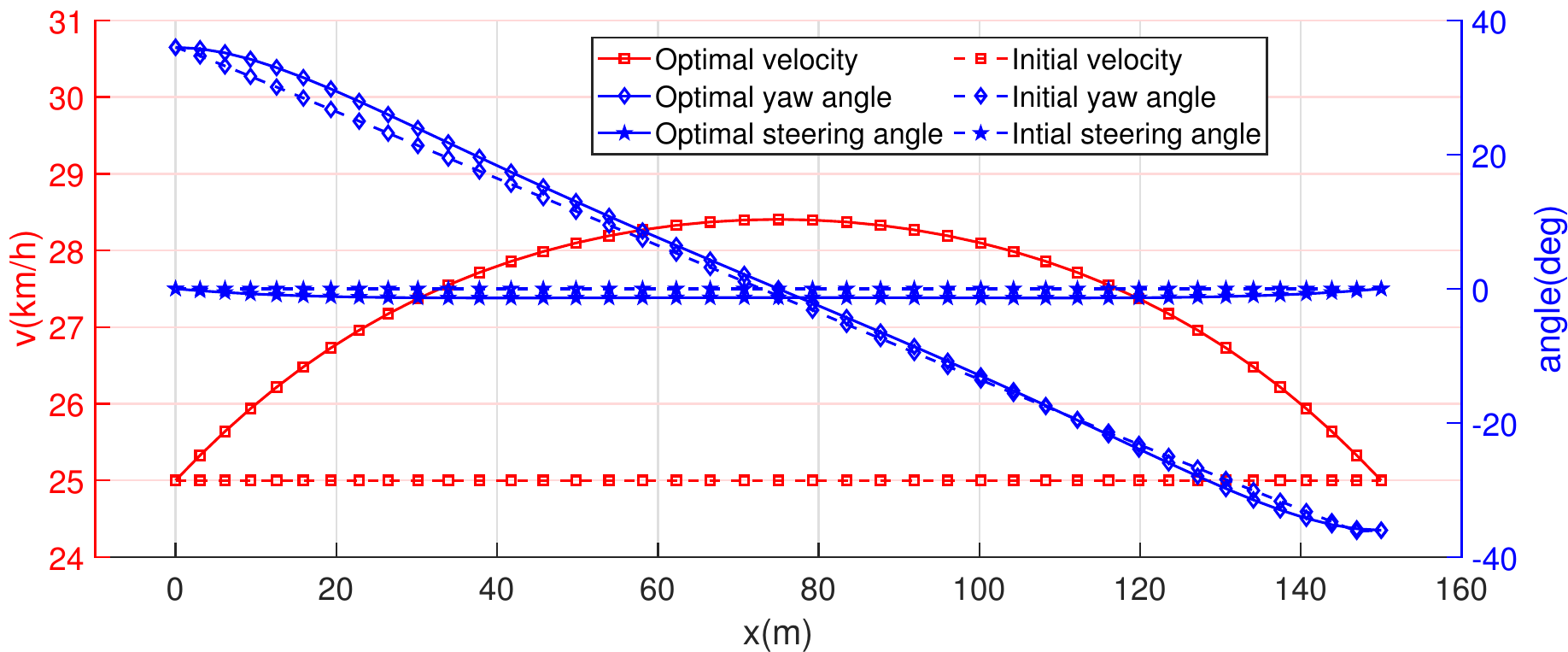}%
\label{fig2d}}
\caption{The resulting trajectory and states in different scenarios. In subplot (a) and (c), the curved lane is surrounded by two curves depicted in red solid. The boundary of obstacle $\mathcal O_1$ is denoted as the inner red curve. Positions of moving obstacle $\mathcal O_2$ are illustrated using red polygons. Positions of the overtaking process of the ego vehicle $\mathcal B$ are illustrated using blue polygons. In Scenario A and B, the ego vehicle $\mathcal B$ plans to overtake the leading obstacle $\mathcal O_2$ in the constrained curved lane. In Scenario A, the starting state $\boldsymbol{\xi_{\tx{init}}}$ is $\begin{bmatrix}\SI{0}{m}, \SI{0}{m}, \SI{36}{\degree}, \SI{25}{km/h}, \SI{0}{\degree}\end{bmatrix}$ and the ending state $\boldsymbol{\xi_{\tx{final}}}$ is $\begin{bmatrix}\SI{150}{m}, \SI{0}{m},\SI{-36}{\degree},\SI{30}{km/h},\SI{0}{\degree}\end{bmatrix}$. In Scenario B, the starting state $\boldsymbol{\xi_{\tx{init}}}$ is $\begin{bmatrix}\SI{0}{m}, \SI{0}{m}, \SI{36}{\degree}, \SI{25}{km/h}, \SI{0}{\degree}\end{bmatrix}$ and the ending state $\boldsymbol{\xi_{\tx{final}}}$ is $\begin{bmatrix}\SI{150}{m}, \SI{0}{m},\SI{-36}{\degree},\SI{25}{km/h},\SI{0}{\degree}\end{bmatrix}$. The initial paths and the planned optimal paths are depicted in subplot (a) and (c). The designed initial states and the resulting states are depicted in subplot (b) and (d). } 
\label{fig2}
\end{figure*}

\subsubsection{Problem description}
It is common for a vehicle to drive in curved lanes, e.g., trucks safely traveling on mountain roads and cars racing in closed-loop tracks \cite{ref31,ref44,ref45}. These vehicles are expected to reside within a narrow curved lane, not colliding with moving obstacles. Typically, it is difficult to overtake a moving obstacle for intelligent vehicles in a narrow curved environment, especially for a long-size vehicle. To approach this issue, we use separating constraints \eqref{eq12}, \eqref{eq21} and containing constraints \eqref{eq28} in the proposed planning framework. 

\subsubsection{Case description}
In the examples of Fig.~\ref{fig2}(a) and Fig.~\ref{fig2}(c), the driving environment $\mathcal W$ is modeled by an ellipse, depicted in outer red solid line, denoted as {$\mathcal W:=\left\{\boldsymbol{s} \in \mathbb{R}^2 \mid\left(\boldsymbol{s}-\boldsymbol{e_1}\right)^{\top} E_1 \left(\boldsymbol{s}-\boldsymbol{e_1}\right) \leq 1\right\},$} where $ E_1= \text{diag}(1/129.5^2,1/129.5^2)$, $\boldsymbol{e_1}=[75,-100]^\top$. The obstacle $\mathcal O_1$ is modeled by an ellipse denoted as $\mathcal O_1:=\left\{\boldsymbol{s} \in \mathbb{R}^2 \mid\left(\boldsymbol{s}-\boldsymbol{e_2}\right)^{\top} E_2\left(\boldsymbol{s}-\boldsymbol{e_2}\right) \leq 1\right\}$, where $ E_2= \text{diag}(1/123.5^2,1/123.5^2)$, $\boldsymbol{e_2}=[75,-100]^\top$. It is illustrated by the inner red solid line in the figure. 
It can be noticed that the curved lane is constructed with an intersection of ellipse $\mathcal W$ and ellipse $\mathcal O_1$, i.e., $\mathcal W \cap  \mathcal O_1$. Another moving obstacle is denoted as a polygon $\mathcal O_2$ of size $\SI{4.6}{m} \times \SI{2}{m}$. We assume obstacle $\mathcal O_2$ moves slowly in a predefined path, which is defined as a uniform circular motion with \SI{1.32}{\degree/s}. It is illustrated by a red solid polygon in Fig.~\ref{fig2}(a) and Fig.~\ref{fig2}(c). The ego vehicle $\mathcal B$ is modeled as a polygon of the same size as $\mathcal O_2$, depicted in blue. The matrix of vertices in $\mathcal B(\boldsymbol{\xi})$ is represented by $V(\boldsymbol{\xi})\in\mathbb R^{2\times 4}$.
In this overtaking planning problem, the ego vehicle should reside in the curved lane and not collide with obstacles, represented as  $\mathcal B \cap \mathcal O_1 =\emptyset $, $\mathcal B \cap \mathcal O_2 =\emptyset $, and  $\mathcal B \subset \mathcal W$, which are approached by using the proposed formulations \eqref{eq21}, \eqref{eq12},  and \eqref{eq28}. To support the efficacy of the proposed methods, we design two different overtaking scenarios, i.e., the vehicle keeps accelerating in Scenario A while in Scenario B the vehicle first speeds up and then slows down to its original speed.

\begin{figure*}[!t]
\centering
\subfloat[Tangent hyperplanes in scenario A]{\includegraphics[width=0.5\textwidth]{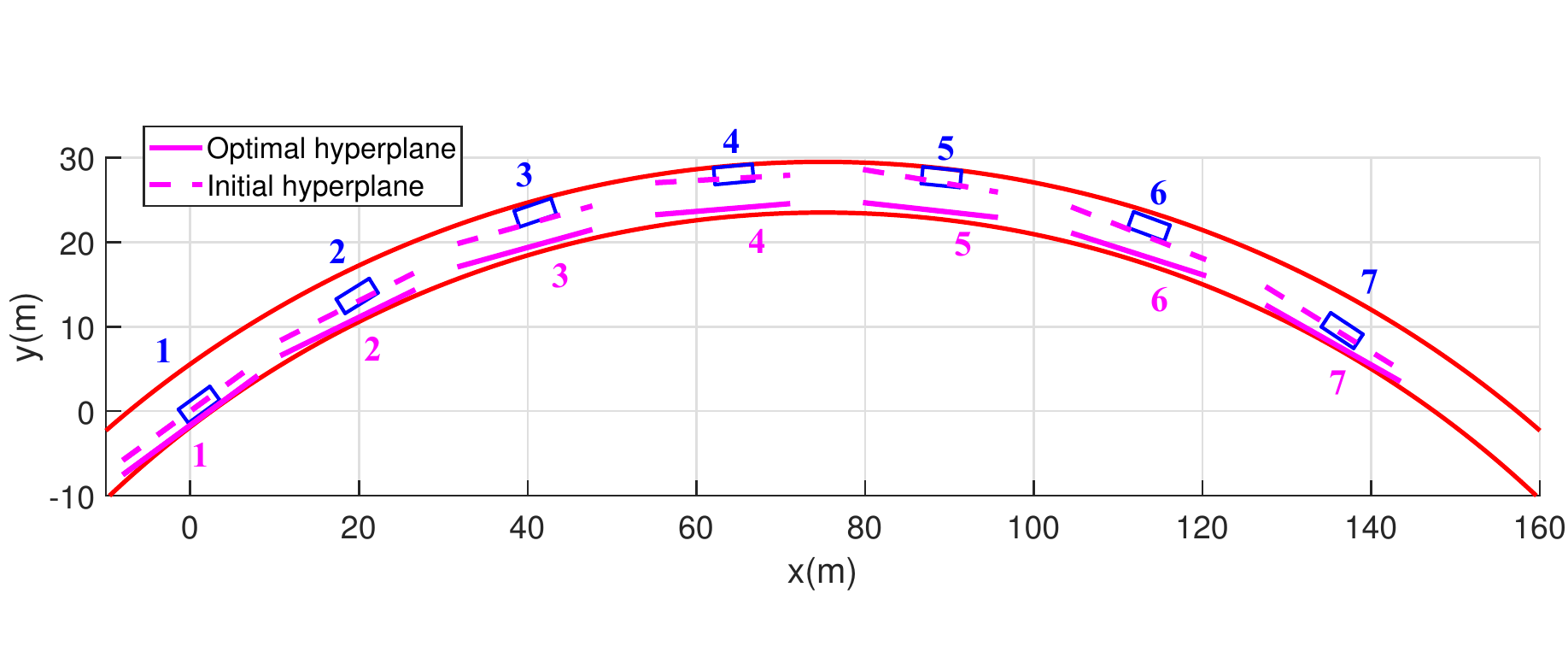}%
\label{fig4a}}
\hfil
\subfloat[Orthogonal hyperplanes in scenario A]{\includegraphics[width=0.5\textwidth]{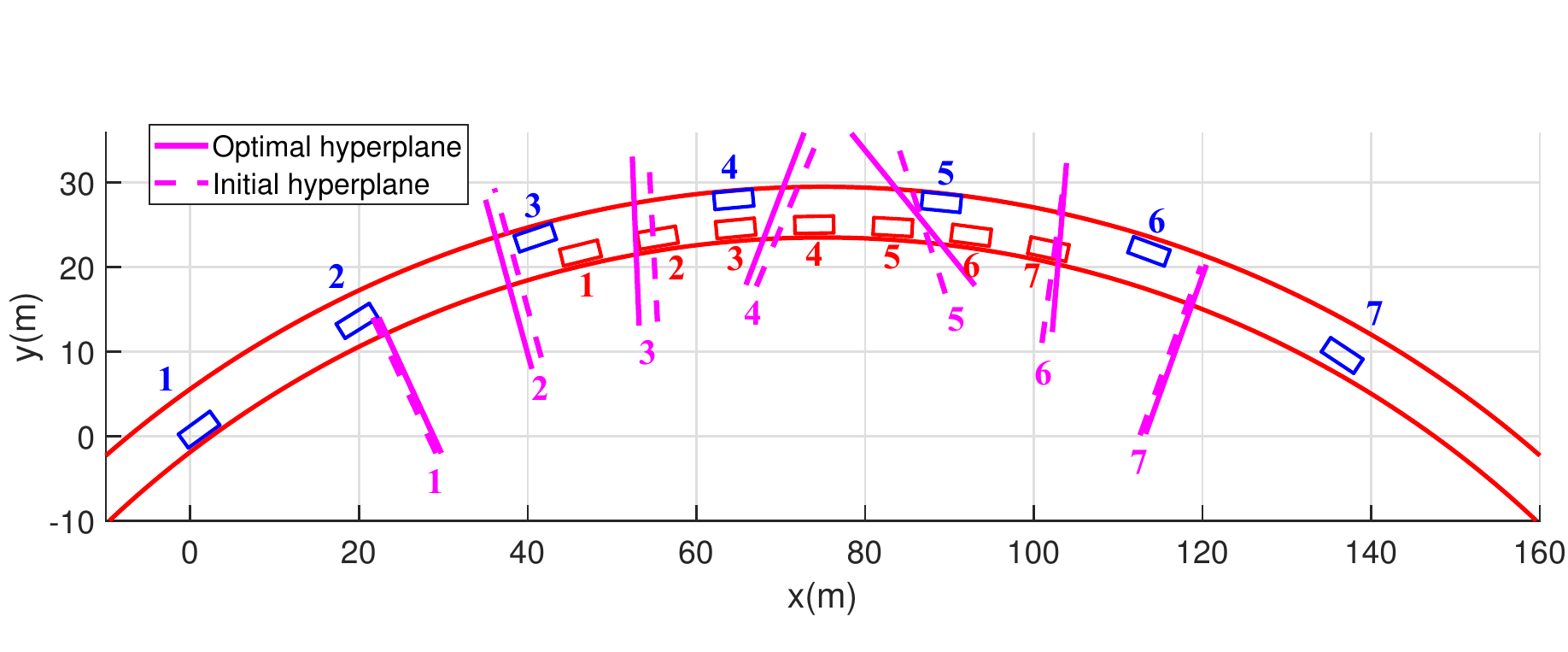}%
\label{fig4b}}
\hfil
\subfloat[Tangent hyperplanes in scenario B]{\includegraphics[width=0.5\textwidth]{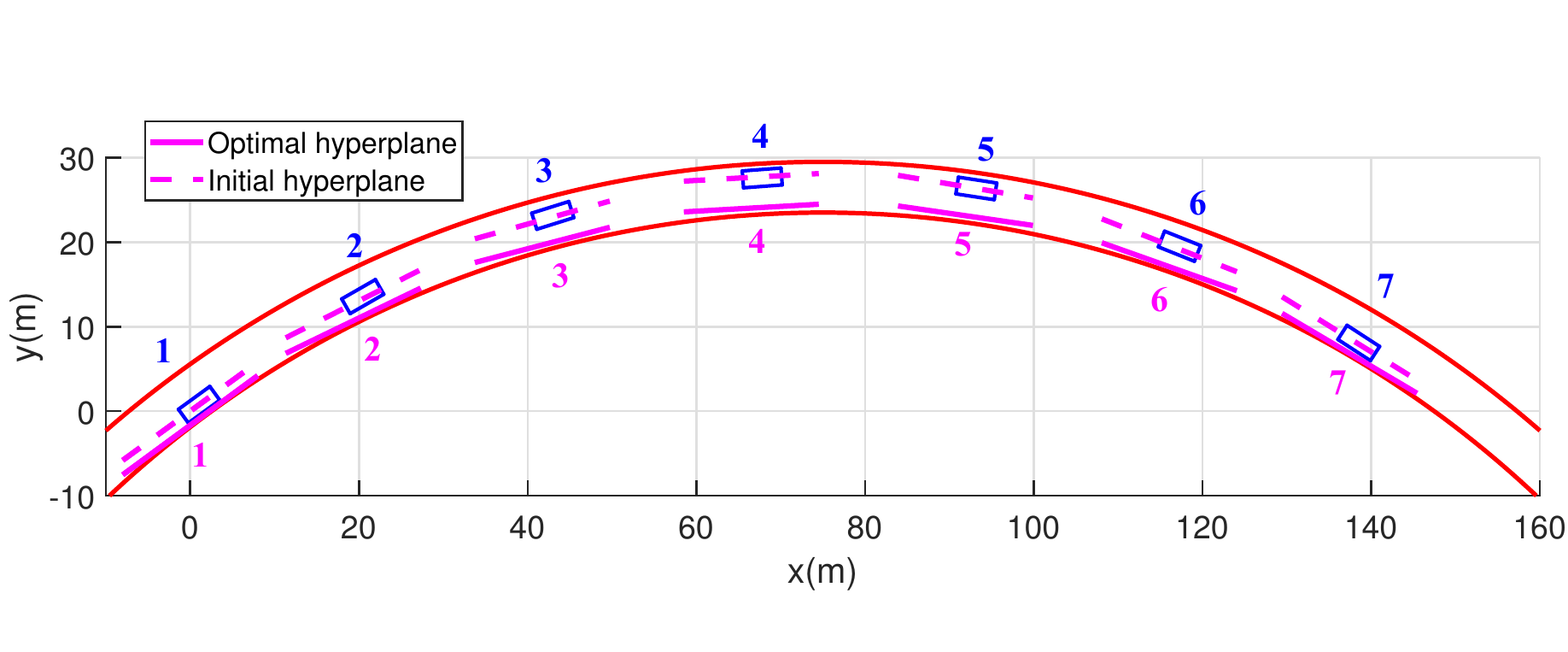}%
\label{fig4c}}
\hfil
\subfloat[Orthogonal hyperplanes in scenario B]{\includegraphics[width=0.5\textwidth]{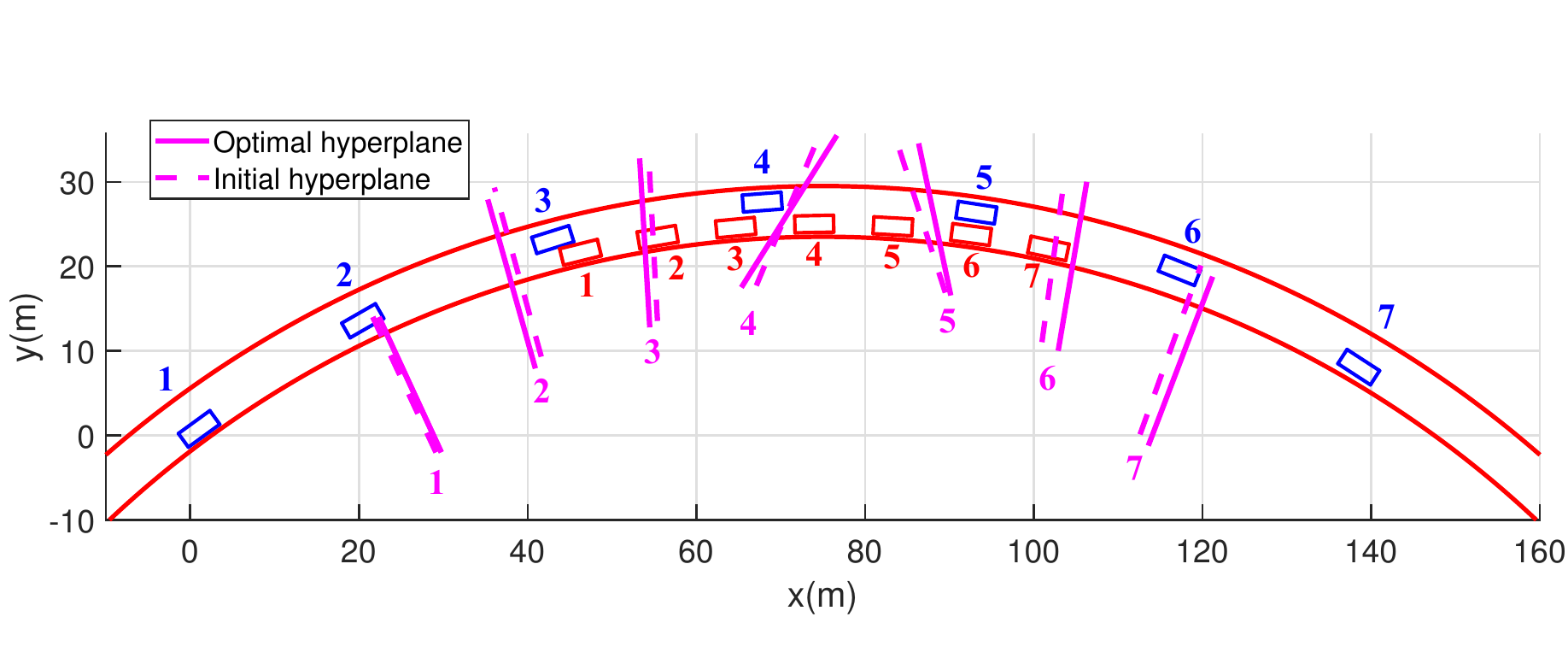}%
\label{fig4d}}
\caption{The hyperplanes to initialize auxiliary variables in constraints \eqref{eq21} and \eqref{eq12} are depicted in magenta dotted lines. The resulting hyperplanes are depicted in magenta solid lines. In (b) and (d) positions of moving obstacle $\mathcal O_2$ are depicted in red polygons. Subplots (a)-(b) and subplots (c)-(d) illustrate positions of the initial hyperplanes and resulting hyperplanes in seven samples in Scenario A and Scenario B, respectively. It can be seen that in each sample the initial hyperplane is close to the resulting hyperplane and the resulting hyperplane can separate the vehicle from obstacle $\mathcal O_1$ in (a) and (c) and separate the vehicle from moving obstacle $\mathcal O_2$ in (b) and (d).} 
\label{fig4}
\end{figure*}

\subsubsection{Vehicle model and cost function}  \label{sec7.2.3}
As Section~\ref{sec7.1.5} does, we use a simplified vehicle model of \eqref{eq37}, i.e., with the reduced state vector $\boldsymbol{\xi} = \left[x,y,\theta,v,\delta \right]^{\top}$ and control vector
$\boldsymbol{u}\left(s\right) = \left[a,\omega\right]^{\top} $, where $\theta$ represents the yaw angle with respect to the horizontal axle. The feasible values of $\theta,v,\delta$ are limited to $|\theta| \leq \SI{85}{\degree}$, $ \frac{5}{3.6} \si{m/s^2} \leq v \leq \frac{60}{3.6} \si{m/s^2}$, $|\delta| \leq 45 \si{\degree}$. Additionally, the feasible inputs are given by $\SI{-1}{m/s^2} \leq a\leq \SI{2}{m/s^2}$ and $|\omega|\leq \SI{5}{\degree/s^2}$.

In this overtaking case, we expect the vehicle to drive smoothly to a target position subject to a given time \( t_{f}\). A discretization procedure can then similarly be applied as proposed in Section~\ref{sec7.1.2}, i.e., the given time \( t_{f}\) is divided into \( {k}_{f}+1\) parts.  In the objective function,  the control inputs $a$ and $\omega$ are penalized as $\sum_{k=0}^{k_f-1}  \left\Vert \boldsymbol{u}\left(k\right)\right\Vert_{Q}^{2}$, with $Q={\rm diag}(10, 20)$. Thus, the OCP \eqref{eq2} is formulated as the NLP 
{\allowdisplaybreaks
\begin{subequations}\label{eq35}
\begin{align}
\underset {\boldsymbol{\xi}, \boldsymbol{u},(\cdot)}{\text{min}} & \ \sum_{k=0}^{k_f-1} \left\Vert \boldsymbol{u}\left(k\right)\right\Vert_{Q}^{2} \label{eq35a}\\
\text {s.t.}  \ \ & \boldsymbol{\xi}(0)=\boldsymbol{\xi_{\tx{init}}},\ \boldsymbol{\xi}(k_\tx{f})=\boldsymbol{\xi_{\tx{final}}}, \label{eq35b}\\
 & \boldsymbol{\xi}(k+1) =\tilde{f}\left(\boldsymbol{\xi}(k),\boldsymbol{u}(k)\right), \label{eq35c}\\
 & \boldsymbol{\xi}(k) \in \mathcal{X},\ \boldsymbol{u}(k) \in \mathcal{U}, \label{eq35d}\\
 & (\ref{eq12}), (\ref{eq21}), (\ref{eq28}), \label{eq35e}
\end{align}
\end{subequations}}%
where symbol $\tilde{f}$ denotes the discretized symbol of vehicle model \eqref{eq37} with removing state $\theta_2$ related to the trailer. The 4-th order Runge-Kutta method is used to express the state evolution of \eqref{eq35c}.

The 4-th order polynomial function is applied to generate a trajectory as the initial guesses \cite{ref80,ref81} in states $(x,y,\theta)$ . The initial guess for $v$ is a linear function from the initial to the final value. The rest of the states and inputs are set as zero. Let the initial separating hyperplane, which is used to initialize auxiliary variables $\boldsymbol{\lambda}$, $\mu$ in formulation \eqref{eq21}, be initialized with the hyperplane tangent to the polynomial trajectory at each sample. Likewise in the parking case, let the hyperplane, which is to initialize $\boldsymbol{\lambda}$, $\mu$ in formulation \eqref{eq12}, be orthogonal to the line crossing the sample on the polynomial trajectory and the centroid of an obstacle, and $\gamma=0.5$ is set. In this overtaking case,  the horizon time is set to $\SI{21}{s}$ and $k_\tx{f}=40$. Other NLP configurations are the same as that of the parking case.

\subsubsection{Solution}
The resulting optimal trajectory and states are visually interpreted in Fig.~\ref{fig2}. The primary observation in Fig.~\ref{fig2}(a) and Fig.~\ref{fig2}(c) is that the resulting optimal paths are smooth and collision-free. In Fig.~\ref{fig2}(a) derived from Scenario A, since a higher speed is acquired, a larger turning radius of the car is needed. Because this curved lane is narrow, it can be seen that the vehicle keeps on close to the environment boundaries when driving during [\SI{40}{m},\SI{100}{m}], while each vertex of the vehicle does not cross the boundaries. The proposed containing constraint \eqref{eq28} contributes to this. On the contrary, in Fig.~\ref{fig2}(c) from Scenario B, the vehicle mostly drives along the centerline, not keeping close to the environment boundaries. Additionally, in these two scenarios, the vehicle completes the overtaking task subject to a given time, not colliding with the moving obstacle $\mathcal O_2$, and also during this whole process the vehicle does not enter into obstacle $\mathcal O_2$. The proposed separating constraint \eqref{eq12} contributes to this.

Another observation in Fig.~\ref{fig2}(a)--Fig.~\ref{fig2}(d) is that the resulting optimal paths and states are close to the designed initial guesses, indicating that the given initial guesses can improve the efficiency and solution quality in solving the nonconvex NLP \eqref{eq35}.  Additionally, from the state profiles of Fig.~\ref{fig2}(b) and Fig.~\ref{fig2}(d) , the steering angle, the velocity, and the yaw angle of the vehicle vary steadily during the traveling process and achieve the expected target states. Specifically, in Scenario A the speed increases steadily from the beginning $\SI{25}{km/h}$ to the expected $\SI{30}{km/h}$ while in Scenario B the speed first increases from $\SI{25}{km/h}$ to $\SI{28.4}{km/h}$ and returns to the expected $\SI{25}{km/h}$ in the end.

\subsubsection{Comparisons in initializing auxiliary variables}
To investigate the computational cost of the proposed formulations \eqref{eq12} and \eqref{eq21} when initialized differently, comparative experiments are conducted under existing initialization methods in \cite{ref1} and the proposed geometry-based methods, respectively. The existing method initializes $\boldsymbol{\lambda}$, in formulations \eqref{eq12} and \eqref{eq21} of NLP \eqref{eq35},  with a positive value to prevent a singularity and $\mu$ is set to zero. The proposed method uses an orthogonal hyperplane to initialize the auxiliary variables in formulation \eqref{eq12} and a tangent hyperplane to initialize the auxiliary variables in formulation \eqref{eq21}. Other configurations of NLP \eqref{eq35} remain unchanged. 

Fig.~\ref{fig4} illustrates the proposed geometry-based initial guesses in Scenario A and Scenario B. It can be seen that the designed tangent or orthogonal hyperplane in each sample is in the neighborhood of the resulting separating hyperplane, showing that the auxiliary variables related to constraints \eqref{eq12} and \eqref{eq21} have been well initialized. This is beneficial to solve the nonconvex NLP \eqref{eq35}, which is supported in Table~\ref{tab2} showing the computational results using different initialization methods. From Table~\ref{tab2}, though using different initial guesses the resulting optimal solutions are exactly the same in each scenario, i.e., the objective value is 0.907 and 0.169, respectively. However, the result of solving time shows that the proposed geometry-based guesses lead to the less computation time than the existing guesses, especially when initializing the variables in constraints \eqref{eq12}. The reason behind this is that the proposed geometry-based method can well initialize  variables to be optimized in NLP \eqref{eq35}. The results in Table~\ref{tab2} illustrates the computational advantages of our methods, which are highly beneficial for efficient collision avoidance.  

\begin{table}[htbp]
  \centering
  \caption{Results in computational effort using different initial guesses.}
  \begin{threeparttable}
    \begin{tabular}{cclll}
    \toprule
Scenario & \multicolumn{1}{l}{\parbox{22mm}{Ways to initialize auxiliary  variables  in (12) and (21)}} & \multicolumn{1}{l}{\parbox{10mm}{Objective ~~value}} & TS\tnote{1}    & \multicolumn{1}{l}{\parbox{10mm}{Solving time(s)}}  \\
    \midrule
 \multirow{4}[2]{*}{Scenario A } &  \parbox{22mm}{Proposed to (12) \& Proposed to  (21)} & 0.907 & 0.168 & \textbf{1.362}\tnote{2} \\
   & \parbox{22mm}{Existing to (12) \& Proposed to (21)}  & 0.907 & 0.168 & 8.314 \\
  & \parbox{22mm}{ Proposed to (12) \& Existing to (21)} & 0.907 & 0.168 & 2.132 \\
    \midrule
  \multirow{4}[2]{*}{Scenario B } & \parbox{22mm}{Proposed to (12) \& Proposed to  (21)} & 0.169 & 0.410 & \textbf{0.962} \\
   & \parbox{22mm}{Existing to (12) \& Proposed to (21)}  & 0.169 & 0.410 & 4.717 \\
  & \parbox{22mm}{ Proposed to (12) \& Existing to (21)} & 0.169 & 0.410 & 1.763   \\
    \bottomrule
    \end{tabular}%
       \begin{tablenotes}   
        \footnotesize    
        \item[1] The value $ \tx{TS}=\sum_{{k}=0}^{{k}_f-1} \left(a(k)^2+\omega(k)^2\right)$ is calculated to double check if the NLP \eqref{eq35} converges to the same local solution under a same objective value.
        \item[2] The bold value denotes the least solving time.     
      \end{tablenotes}
    \end{threeparttable}
  \label{tab2}%
\end{table}%

\section{Conclusion and future work}\label{sec8}
This article investigates optimization-based methods aiming to achieve exact collision avoidance of controlled objects. In the proposed unified planning framework, we formulate smooth collision avoidance constraints by utilizing the hyperplane separation theorem to separate a pair of convex sets, and S-procedure and geometrical methods to contain a convex set inside another convex set. The proposed collision avoidance formulations are highly beneficial for reducing an NLP problem size and efficient initialization of auxiliary variables, which admits an improved computational performance compared to the state of the art. 

Based on the proposed planning framework, the next work will first focus on the research of the uncertainty from localization errors, system delay, etc. Additionally, since collision formulations using dual methods in \cite{ref3} (with higher complexity than the proposed collision formulations using the hyperplane separation theorem) have successfully been implemented in real-time \cite{ref32} and even in real-world scenarios \cite{ref66, ref67}, it is possible to execute the proposed methods in real-time by sequential quadratic programming implementations \cite{ref68, ref69}. Real-time control implementation of the proposed method will also be considered in future studies. 

{\appendices
\section{Extensions of separating constraints}
\subsection{Hyperplane with bounded slope or displacement}\label{appendix1}
In Proposition~\ref{prop2}, Proposition~\ref{prop3}, and Proposition~\ref{prop4}, we propose using $N+1$ auxiliary variables to describe a separating hyperplane. The three propositions are sufficient and necessary conditions of $\mathcal B_m \cap \mathcal O_n=\emptyset$.  In principle, it is possible to use fewer variables to obtain sufficient conditions of $\mathcal B_m \cap \mathcal O_n=\emptyset$. For instance, consider Proposition~\ref{prop2} under $N=2$. According to the proposition, 3 auxiliary variables are needed to construct a hyperplane as
\begin{equation}\label{eq60}
\begin{aligned}
\mathcal B_m \cap \mathcal O_n=\emptyset \iff  \exists \boldsymbol{\lambda} \in \mathbb{R}^2 , \mu \in \mathbb{R}: \ \ \ \ \ \ \ \ \  \ \ \ \   \\ 
  \begin{array}{c}
\boldsymbol{\lambda}^\top V_{P_1} \geq \mu \boldsymbol{1}^\top, \boldsymbol{\lambda}^\top V_{P_2} \leq \mu \boldsymbol{1}^\top, \left\|\boldsymbol{\lambda}\right\| > 0.
\end{array}
\end{aligned}
\end{equation}

In certain cases, it is possible to use only 2 auxiliary variables, e.g., if it is known that $\mu\neq0$, the inequalities can be normalized by $\mu$ and a sufficient condition of $\mathcal B_m \cap \mathcal O_n=\emptyset$ can be formulated as
\begin{equation}\label{eq61}
\begin{aligned}
\mathcal B_m \cap \mathcal O_n=\emptyset \ \Longleftarrow \ \exists \boldsymbol{\lambda} \in \mathbb{R}^2: \ \ \ \ \ \ \ \ \  \ \ \ \  \ \ \ \  \ \ \ \ \  \ \ \  \\ 
  \begin{array}{c}
\boldsymbol{\lambda}^\top V_{P_1} \geq \boldsymbol{1}^\top, \boldsymbol{\lambda}^\top V_{P_2} \leq \boldsymbol{1}^\top, \left\|\boldsymbol{\lambda}\right\| > 0.
\end{array}
\end{aligned}
\end{equation}
Alternatively, if it is known that both horizontal and vertical hyperplane slopes are not occurring in the same problem, then one of the $\boldsymbol{\lambda}$ is not zero and the inequality can be normalized with that value. For e.g., such a condition can be written as
\begin{equation}\label{eq62}
\begin{aligned}
\mathcal B_m \cap \mathcal O_n=\emptyset \ \Longleftarrow \ \exists \lambda \in \mathbb{R}, \mu \in \mathbb{R}: \ \ \ \ \ \ \ \ \  \ \ \ \  \ \ \ \ \ \ \ \\ 
  \begin{array}{c}
\begin{bmatrix} \lambda & 1
\end{bmatrix} V_{P_1} \geq \mu\boldsymbol{1}^\top,\begin{bmatrix} \lambda & 1
\end{bmatrix} V_{P_2} \leq \mu\boldsymbol{1}^\top, \lambda^2 > 0.
\end{array}
\end{aligned}
\end{equation}
 
Thus, using fewer variables can simplify the Proposition~\ref{prop2}, Proposition~\ref{prop3}, and Proposition~\ref{prop4}, which can further reduce problem size and computational load.

\subsection{A point-mass controlled object}\label{appendix2}
Apart from the investigation of collision avoidance between polytopes and ellipsoids in this article, the hyperplane separation theorem can readily be applied to other compact sets. For instance, sometimes the controlled object $\mathcal B$ is regarded as a point-mass system. Let the position of a point $p$ be denoted as $\boldsymbol{p} \in \mathbb R ^N$. Let set $\mathcal O_n$ be denoted using polytope $\mathcal P(P, \boldsymbol p, \boldsymbol s)$ as \eqref{eq3}. Then, 
\begin{equation}\label{eq50}
\begin{aligned}
p \cap \mathcal O_n=\emptyset \iff  \exists \boldsymbol{\lambda} \in \mathbb{R}^N , \mu \in \mathbb{R}: \ \ \ \ \ \ \ \ \  \ \ \ \   \\ 
  \begin{array}{c}
\boldsymbol{\lambda}^\top \boldsymbol{p} \geq \mu, \boldsymbol{\lambda}^\top V_P \leq \mu \boldsymbol{1}^\top, \left\|\boldsymbol{\lambda}\right\| > 0.
\end{array}
\end{aligned}
\end{equation}
In formulation \eqref{eq50}, the number of additional variables i.e., $\boldsymbol{\lambda}, \mu$, does not depend on the complexity of polytope $\mathcal O_n$, so formulation \eqref{eq50} also admits a reduced problem size compared to formulations in Proposition 1 of \cite{ref3} and Section III-A-1 of \cite{ref6}.

It is possible to even further reduce the number of variables to formulate \eqref{eq50}. The hyperplane for separation can be attached to the vehicle point
$p$, i.e., $\boldsymbol{\lambda}^\top \boldsymbol{p} = \mu$, and a sufficient condition of $p \cap \mathcal O_n=\emptyset$ can be formulated as
\begin{equation}\label{eq51}
\begin{aligned}
p \cap \mathcal O_n=\emptyset \ \Longleftarrow \ & \ \exists \boldsymbol{\lambda} \in \mathbb{R}^N:\\
 & \begin{array}{c}
\boldsymbol{\lambda}^\top V_P \leq \boldsymbol{\lambda}^\top \boldsymbol{p} \boldsymbol{1}^\top, \left\|\boldsymbol{\lambda}\right\| > 0.
\end{array}
\end{aligned}
\end{equation}

}
\section*{Acknowledgments}
The authors thank Dr. Jian Zhou from Link\"oping University for valuable suggestions.

\vfill


\begin{thebibliography}{99}
\bibliographystyle{IEEEtran}

\bibitem{ref1}
C. Dietz, S. Albrecht, A. Nurkanović and M. Diehl, "Efficient Collision Modelling for Numerical Optimal Control," in \emph{Proc. Eur. Control Conf. (ECC)}, Bucharest, Romania, pp. 1-7, 2023.
\bibitem{ref2}
 J. Karlsson, N. Murgovski, and J. Sjoberg, "Computationally Efficient Autonomous Overtaking on Highways," \emph{IEEE Trans. Intell. Transp. Syst.}, vol. 21, no. 8, pp. 3169-3183, 2020.

\bibitem{ref3}
X. Zhang, A. Liniger, and F. Borrelli, "Optimization-Based Collision Avoidance," \emph{IEEE Trans. Contr. Syst. Technol.}, vol. 29, no. 3, pp. 972-983, 2021.

\bibitem{ref4}
J. Schulman et al., "Motion planning with sequential convex optimization and convex collision checking," \emph{Int. J. Robot. Autom.}, vol. 33, no. 9, pp. 1251-1270, 2014.

\bibitem{ref61}
V. Sunkara, A. Chakravarthy and D. Ghose, "Collision Avoidance of Arbitrarily Shaped Deforming Objects Using Collision Cones," \emph{IEEE Robot. Autom. Lett.}, vol. 4, no. 2, pp. 2156-2163, 2019.


\bibitem{ref62}
S. Haddadin, A. De Luca and A. Albu-Schäffer, "Robot Collisions: A Survey on Detection, Isolation, and Identification," \emph{IEEE Trans. Robot.}, vol. 33, no. 6, pp. 1292-1312, 2017.
 
\bibitem{ref6}
S. Helling and T. Meurer, "Dual Collision Detection in Model Predictive Control Including Culling Techniques," \emph{IEEE Trans. Control. Syst. Technol.}, early access, 2023, doi: 10.1109/TCST.2023.3259822.

\bibitem{ref7}
B. Li \emph{et al.}, "Tractor-Trailer Vehicle Trajectory Planning in Narrow Environments With a Progressively Constrained Optimal Control Approach," \emph{IEEE Trans. Intell. Veh.}, vol. 5, no. 3, pp. 414-425, 2020.

\bibitem{ref70}
B. Li \emph{et al.}, "Optimization-Based Trajectory Planning for Autonomous Parking With Irregularly Placed Obstacles: A Lightweight Iterative Framework," \emph{IEEE Trans. Intell. Transp. Syst.}, vol. 23, no. 8, pp. 11970-11981, 2022.

\bibitem{ref8}
J. Schulman \emph{et al.}, "Motion planning with sequential convex optimization and convex collision checking," \emph{Int. J. Robot. Autom.}, vol. 33, no. 9, pp. 1251-1270, 2014.

\bibitem{ref9}
P.F. Lima, "Optimization-based motion planning and model predictive control for autonomous driving: With experimental evaluation on a heavy-duty construction truck," Ph.D. dissertation, KTH Royal Institute of Technology, 2018.

\bibitem{ref82}
J. Zhou, B. Olofsson and E. Frisk, "Interaction-Aware Motion Planning for Autonomous Vehicles with Multi-Modal Obstacle Uncertainty Predictions," \emph{IEEE Tran. Intell. Veh.}, early access, 2023, doi:10.1109/TIV.2023.3314709.

\bibitem{ref83}
J. Zhou, B. Olofsson and E. Frisk, "Interaction-Aware Moving Target Model Predictive Control for Autonomous Vehicles Motion Planning," in \emph{Proc. Eur. Contr. Conf. (ECC)}, London, United Kingdom, pp. 154-161, 2022.

\bibitem{ref63}
J. Hu, H. Zhang, L. Liu, X. Zhu, C. Zhao and Q. Pan, "Convergent Multiagent Formation Control With Collision Avoidance," \emph{IEEE Trans. Robot.}, vol. 36, no. 6, pp. 1805-1818, 2020.

\bibitem{ref10}
D. Ioan, I. Prodan, S. Olaru, F. Stoican, and S. Niculescu, "Mixed-integer programming in motion planning," \emph{Annu. Rev. Contr.}, vol. 51, pp. 65-87, 2021.

\bibitem{ref11}
I. E. Grossmann,  "Review of Nonlinear Mixed-Integer and Disjunctive Programming Techniques," \emph{Optim. Eng.}, vol. 3, pp. 227–252, 2002.


\bibitem{ref12}
S. M. LaValle \emph{planning algorithms}. Cambridge University Press. 2006.

\bibitem{ref13}
R. Deits and R. Tedrake, "Efficient mixed-integer planning for UAVs in cluttered environments," in \emph{Proc. IEEE Int. Conf. Robot. Autom. (ICRA)}, Seattle, WA, USA, pp. 42-49, 2015.

\bibitem{ref14}
D. Mellinger, A. Kushleyev and V. Kumar, "Mixed-integer quadratic program trajectory generation for heterogeneous quadrotor teams," in \emph{Proc. IEEE Int. Conf. Robot. Autom. (ICRA)}, Saint Paul, MN, USA, pp. 477-483, 2012.


\bibitem{ref15}
I. Prodan, F. Stoican, S. Olaru and S.I. Niculescu, \emph{Mixed-Integer Representations in Control Design Mathematical Foundations and Applications}. Springer. 2015.

\bibitem{ref16}
J. P. Vielma, "Mixed Integer Linear Programming Formulation Techniques," \emph{Siam. Rev.}, vol. 57, no. 1, pp. 3–57, 2015.

\bibitem{ref17}
X. Wang, M. Kloetzer, C. Mahulea and M. Silva, "Collision avoidance of mobile robots by using initial time delays," in \emph{Proc. IEEE Conf. Decis. Control (CDC)}, Osaka, Japan, pp. 324-329, 2015.

\bibitem{ref18}
I. Haghighi, S. Sadraddini and C. Belta, "Robotic swarm control from spatio-temporal specifications," in \emph{Proc. IEEE Conf. Decis. Control (CDC)}, Las Vegas, NV, USA, pp. 5708-5713, 2016.

\bibitem{ref19}
J. Karlsson, N. Murgovski and J. Sjoberg, "Comparison between mixed-integer and second order cone programming for autonomous overtaking," in \emph{Proc. Eur. Control Conf. (ECC)}, Limassol, Cyprus, pp. 386-391, 2018.

\bibitem{ref20}
F. Molinari, N. N. Anh, and L. Del Re, "Efficient mixed integer programming for autonomous overtaking," in \emph{Proc. Am. Control Conf. (ACC)}, Seattle, WA, USA, pp. 2303-2308, 2017.

\bibitem{ref21}
H. P. Williams, \emph{Model Building in Mathematical Programming}. Wiley. 2013.

\bibitem{ref22}
X. Zhang, A. Liniger, A. Sakai and F. Borrelli, "Autonomous Parking Using Optimization-Based Collision Avoidance," " in \emph{Proc. IEEE Conf. Decis. Contr. (CDC)}, Miami, FL, USA, pp. 4327-4332, 2018.

\bibitem{ref23}
I. Ballesteros-Tolosana, S. Olaru, P. Rodriguez-Ayerbe, G. Pita-Gil and R. Deborne, "Collision-free trajectory planning for overtaking on highways," in \emph{Proc. IEEE Conf. Decis. Contr. (CDC)}, Melbourne, VIC, Australia, pp. 2551-2556, 2017.


\bibitem{ref24}
Rubens J. M. Afonso, Roberto K. H. Galvão and K. H. Kienitz, "Reduction in the number of binary variables for inter-sample avoidance in trajectory optimizers using mixed-integer linear programming," \emph{Int. J. Robust Nonlinear Control.}, vol. 26, pp. 3662–3669, 2016.

\bibitem{ref25}
F. Stoican, I. Prodan and S. Olaru, "Hyperplane arrangements in mixed-integer programming techniques. Collision avoidance application with zonotopic sets," in \emph{Proc. Eur. Control Conf. (ECC)}, Zurich, Switzerland, pp. 3155-3160, 2013.

\bibitem{ref26}
A. Ganesan, S. Gros, and N. Murgovski, "Numerical Strategies for Mixed-Integer Optimization of Power-Split and Gear Selection in Hybrid Electric Vehicles," \emph{IEEE Trans. Intell. Transp. Syst.}, vol. 24, no. 3, pp. 3194-3210, 2023.

\bibitem{ref27}
C. Kirches, \emph{Fast Numerical Methods for Mixed-Integer Nonlinear Model-Predictive Control}. Springer. 2010.


\bibitem{ref28}
H. Febbo, J. Liu, P. Jayakumar, J. L. Stein and T. Ersal, "Moving obstacle avoidance for large, high-speed autonomous ground vehicles," in \emph{Proc. Am. Control Conf. (ACC)}, Seattle, WA, USA, pp. 5568-5573,  2017.

\bibitem{ref29}
U. Rosolia, S. De Bruyne, and A. G. Alleyne, “Autonomous vehicle control: A nonconvex approach for obstacle avoidance,” \emph{IEEE Trans. Control. Syst. Technol.}, vol. 25, no. 2, pp. 469–484, 2016.

\bibitem{ref30}
M. Geisert and N. Mansard, "Trajectory Generation for Quadrotor Based Systems using Numerical Optimal Control, " in \emph{Proc. IEEE Int. Conf. Robot. Autom. (ICRA)}, Stockholm, Sweden, pp. 2958-2964, 2016.

\bibitem{ref31}
J. Zeng, B. Zhang, and K. Sreenath, “Safety-critical model predictive control with discrete-time control barrier function,” in \emph{Proc. Am. Control Conf. (ACC)}, New Orleans, LA, USA, pp. 3882–3889, 2021.

\bibitem{ref32}
A. Thirugnanam, J. Zeng and K. Sreenath, "Safety-Critical Control and Planning for Obstacle Avoidance between Polytopes with Control Barrier Functions, " in \emph{Proc. IEEE Int. Conf. Robot. Autom. (ICRA)}, Philadelphia, PA, USA,  pp. 286-292, 2022.

\bibitem{ref110}
B. Li et al., "Optimization-Based Trajectory Planning for Autonomous Parking With Irregularly Placed Obstacles: A Lightweight Iterative Framework," \emph{ IEEE Trans. Intell. Transp. Syst.}, vol. 23, no. 8, pp. 11970-11981, 2022.

\bibitem{ref111}
R. Chai, A. Tsourdos, A. Savvaris, S. Chai, Y. Xia and C. L. P. Chen, "Design and Implementation of Deep Neural Network-Based Control for Automatic Parking Maneuver Process," \emph{IEEE Trans. Neural Netw. Learn. Syst.}, vol. 33, no. 4, pp. 1400-1413, 2022.

\bibitem{ref112}
B. Li et al., "Tractor-Trailer Vehicle Trajectory Planning in Narrow Environments With a Progressively Constrained Optimal Control Approach," \emph{IEEE Trans. Intell. Veh.}, vol. 5, no. 3, pp. 414-425, 2020.



\bibitem{ref64}
J. Fan, N. Murgovski, and J. Liang, " Efficient collision avoidance for autonomous vehicles in polygonal domains," 2023, arXiv:2308.09103v2. Available: https://arxiv.org/abs/2308.09103.

\bibitem{ref33}
J. Guthrie, "A Differentiable Signed Distance Representation for Continuous Collision Avoidance in Optimization-Based Motion Planning," in \emph{Proc. IEEE Conf. Decis. Control (CDC)}, Cancun, Mexico, pp. 7214-7221, 2022.

\bibitem{ref60}
Y. Zheng and K. Yamane, "Generalized Distance Between Compact Convex Sets: Algorithms and Applications," \emph{ IEEE Trans. Robot.}, vol. 31, no. 4, pp. 988-1003, 2015.

\bibitem{ref34}
M. Gerdts, R. Henrion,  D. Hömberg, and C. Landry. "Path planning and collision avoidance for robots," \emph{Numer. Algebra, Contr. Optim.}, vol. 2, no. 3, pp. 437-463, 2012.

\bibitem{ref35}
M. Lutz and T. Meurer, "Efficient Formulation of Collision Avoidance Constraints in Optimization Based Trajectory Planning and Control," in \emph{Proc. IEEE Conf. Contr. Technol. Appl. (CCTA)}, San Diego, CA, USA, pp. 228-233, 2021.


\bibitem{ref36}
V. Soltan, "Support and separation properties of convex sets in finite dimension," \emph{Extracta. Math.}, vol. 36, no. 2, pp. 241-278, 2021.

\bibitem{ref37}
T. Rockafellar, \emph{Convex Analysis}. Princeton University Press. 1970.


\bibitem{ref38}
S. Boyd and L. Vandenberghe, \emph{Convex Optimization}. Cambridge University Press. 2009.


\bibitem{ref39}
J. Tao, S. Scott, W. Joe, and D. Mathieu, "A Geometric Construction of Coordinates for Convex Polyhedra using Polar Duals," in \emph{Proc. Symp. Geom. Process.}, Vienna, Austria, pp. 181-186, 2005.

\bibitem{ref40}
C. A. Hayes, "The Heine-Borel Theorem," \emph{Am. Math. Mon.}, vol. 63, no. 3, pp. 180-182, 1956.

\bibitem{ref91}
M. Kesäniemi and K. Virtanen, "Direct Least Square Fitting of Hyperellipsoids," \emph{IEEE Trans. Pattern Anal. Machine Intell.}, vol. 40, no. 1, pp. 63-76, 2018.

\bibitem{ref90}
J. M. Aldaz, S. Barza, M. Fujii and M. S. Moslehian, "Advances in Operator Cauchy--Schwarz inequalities and their reverses," \emph{Ann. Funct. Anal.}, vol. 6, no. 3, pp. 275--295, 2015.

\bibitem{ref42}
I. Pólik and T. Terlaky, "A Survey of the S-Lemma," \emph{SIAM Review}, vol. 49, pp. 371–418, 2007.

\bibitem{ref43}
R. A. Horn and C. R. Johnson, \emph{Matrix Analysis}. Cambridge University Press. 1985.

\bibitem{ref44}
J. Funke, M. Brown, S. M. Erlien and J. C. Gerdes, "Collision Avoidance and Stabilization for Autonomous Vehicles in Emergency Scenarios," \emph{IEEE Trans. Contr. Syst. Technol.}, vol. 25, no. 4, pp. 1204-1216, 2017.

\bibitem{ref45}
J. Wang, Y. Yan, K. Zhang, Y. Chen, M. Cao and G. Yin, "Path Planning on Large Curvature Roads Using Driver-Vehicle-Road System Based on the Kinematic Vehicle Model," \emph{IEEE Trans. Veh. Technol.}, vol. 71, no. 1, pp. 311-325, 2022.


\bibitem{ref46}
J. A. E. Andersson, J. Gillis, G. Horn, J. B. Rawlings, and M. Diehl, "CasADi: a software framework for nonlinear optimization and optimal control," \emph{Math. Program. Comput.}, vol. 11, no. 1, pp. 1-36, 2019.

\bibitem{ref47}
A. Wachter and L. T. Biegler, "On the implementation of an interior-point filter line-search algorithm for large-scale nonlinear programming," \emph{Math. Program.}, vol. 106, no. 1, pp. 25-57, 2006.


\bibitem{ref49}
A. C. Manav and I. Lazoglu, "A Novel Cascade Path Planning Algorithm for Autonomous Truck-Trailer Parking," \emph{ IEEE Trans. Intell. Transp. Syst.}, vol. 23, no. 7, pp. 6821-6835, 2022.

\bibitem{ref50}
E. Börve,  N. Murgovski, and L. Laine, "Interaction-Aware Trajectory Prediction and Planning in Dense Highway Traffic using Distributed Model Predictive Control," 2023, arXiv:2308.13053v1. Available: https://arxiv.org/abs/2308.13053.

\bibitem{ref51}
X. Ruan et al., "An Efficient Trajectory Planning Method with a Reconfigurable Model for Any Tractor-Trailer Vehicle," \emph{IEEE Trans. Transp. Electr.}, vol. 9, no. 2, pp. 3360-3374, 2021.

\bibitem{ref80}
A. Elawad, N. Murgovski, M. Jonasson and J. Sjoberg, "Road Boundary Modeling for Autonomous Bus Docking Subject to Rectangular Geometry Constraints," in \emph{Proc. Eur. Contr. Conf. (ECC)}, Delft, Netherlands, pp. 1745-1750, 2021.

\bibitem{ref81}
Z. Zhang, L. Zhang, J. Deng, M. Wang, Z. Wang and D. Cao, "An Enabling Trajectory Planning Scheme for Lane Change Collision Avoidance on Highways," \emph{IEEE Trans. Intell. Veh.}, vol. 8, no. 1, pp. 147-158, 2023.


\bibitem{ref66}
S. Gilroy et al., "Autonomous Navigation for Quadrupedal Robots with Optimized Jumping through Constrained Obstacles," in \emph{Proc. IEEE Int. Conf. Autom. Sci. Eng. (CASE)}, Lyon, France, pp. 2132-2139, 2021.

\bibitem{ref67}
X. Shen, E. L. Zhu, Y. R. Stürz and F. Borrelli, "Collision Avoidance in Tightly-Constrained Environments without Coordination: a Hierarchical Control Approach," in \emph{Proc. IEEE Int. Conf. Robot. Autom. (ICRA)}, Xi'an, China,  pp. 2674-2680, 2021.


\bibitem{ref68}
R. Verschueren, \emph{et al.}, "acados—a modular open-source framework for fast embedded optimal control," \emph{Math. Prog. Comp.}, vol. 14, pp. 147–183, 2022.

\bibitem{ref69}
G. Frison and M. Diehl, “HPIPM: A high-performance quadratic programming framework for model predictive control,” \emph{IFACPapersOnLine}, vol. 53, no. 2, pp. 6563–6569, 2020.





\end{thebibliography}
\end{document}